\documentclass[11pt,onecolumn]{article}
\usepackage[top=.75in,bottom=.75in,left=.75in,right=.75in]{geometry}

\usepackage{times}
\usepackage{epsfig}
\usepackage{graphicx}
\usepackage{amsmath}
\usepackage{amssymb}

\usepackage{amsthm}
\usepackage{xspace}
\usepackage[font={small},up,bf]{caption}
\usepackage{multirow}
\usepackage{dsfont}
\usepackage{subfig}
\usepackage{proof}

\usepackage{algorithmic,algorithm}

\usepackage{eucal,cite,verbatim}

\newcommand{\figcenter}[1]{\raisebox{-0.5\height}{#1}}

\expandafter\def\expandafter\normalsize\expandafter{%
\normalsize\setlength\abovedisplayskip{3pt}}

\expandafter\def\expandafter\normalsize\expandafter{%
\normalsize\setlength\belowdisplayskip{3pt}}

\newcommand{\colseperator}{\hspace*{.4in}}

\newtheorem{proposition}{Proposition}

\def\sign{\operatorname*{sign\,}}

\def\tsh{{TSH}\xspace}

\def\x{{\boldsymbol x}}
\def\w{{\boldsymbol w}}

\def\z{{\boldsymbol z}}

\def\bh{{\Phi}}

\def\Z{{\bf  Z}}
\def\Y{{\bf  Y}}
\def\A{{\bf  A}}

\def\bY{{\bf  Y}}

\DeclareMathOperator{\Xcal}{\mathcal{X}}

\def\Real{\mathbb{R}}

\def\T{{\!\top}}

\begin{document}

\title{A General Two-Step Approach to Learning-Based Hashing\thanks{Appearing in
    Proceedings of 14th International Conference on Computer Vision (ICCV) 2013.
}}

\author{
    Guosheng Lin ~~
        Chunhua Shen\thanks{Corresponding author (e-mail: {chunhua.shen@adelaide.edu.au}).
        } ~~
    David Suter ~~
        Anton van den Hengel
        \\
        School of Computer Science,  University of Adelaide, SA 5005, Australia
}

\maketitle
\thispagestyle{empty}

\begin{abstract}
Most existing approaches to hashing apply a single form of hash function, and an
optimization process which is typically deeply coupled to this specific
form. %
This tight coupling restricts the flexibility of the method to respond to the data, and
can result in complex optimization problems that are difficult to solve. Here we propose
a flexible yet simple framework that is able to accommodate different types of loss
functions and hash functions. This framework allows a number of existing approaches to
hashing to be placed in context, and simplifies the development of new problem-specific
hashing methods. Our framework decomposes hashing learning problem into two steps:
hash bit learning and hash function learning based on the learned bits. The first step
can typically be formulated as binary quadratic problems, and the second step can be
accomplished by training standard binary classifiers. Both problems have been extensively
studied in the literature. Our extensive experiments demonstrate that the proposed
framework is effective, flexible and outperforms the state-of-the-art.

\end{abstract}

\section{Introduction}

Recently hashing methods have been widely used for
a variety of applications, but have been particularly successful when applied to
 approximate nearest
neighbour search.  Hashing methods construct a set of hash
functions that map the original high-dimensional data into a
compact binary space.  The resulting binary codes enable fast
similarity search on the basis of the hamming distance between codes. Moreover, compact binary
codes are extremely  efficient for large-scale data storage. Applications in
computer vision include content-based image retrieval,
object recognition \cite{TINY80M}, image matching, etc. %
In general, hash functions are generated with the aim of  preserving some notion of similarity between data points.
One of the seminal approaches in this vain is Locality-Sensitive Hashing (LSH)
\cite{Gionis1999}, which randomly generates hash functions to
approximate cosine similarity. Compared to this data-independent
method, more recent work has focussed on data-dependant approaches for
generating more effective hash functions. In this category,
a number of methods have been proposed including: Spectral
Hashing (SPH)
\cite{weiss2008spectral}, Multi-dimension Spectral Hashing (MDSH)
\cite{MDSH},  Iterative Quantization (ITQ) \cite{gong2012iterative}
Anchor Graph Hashing (AGH) \cite{liu2011hashingGraphs},
Self-taught Hashing (STH) \cite{zhang2010self}, and nonparametric Inductive Hashing on manifolds \cite{CVPR13aShen}.

These methods do not rely on labelled data and are thus categorized unsupervised
hashing methods.
Supervised hashing methods have also been extensively studied: Supervised Hashing with Kernels (KSH) \cite{KSH},
Minimal Loss Hashing (MLH) \cite{norouzi2011minimal}, Supervised Binary Reconstructive
Embeddings (BRE) \cite{kulis2009learning}, Semi-supervised sequential Projection Learning
Hashing (SPLH) \cite{wang2010semi}, and Column Generation Hashing \cite{ICML13a}.

Loss functions %
for hashing
are typically defined on the basis of the hamming distance
(e.g., BRE, MLH) or hamming affinity (e.g., KSH, MDSH, SPLH) of similar and
dissimilar data pairs. Hamming affinity is calculated by the inner
product of two binary codes (a binary code takes a value of
$\{-1,1\}$).
Existing methods thus tend to optimise a single form of hash function, the parameters of which are
directly optimised
against the overall loss function.  The common forms of hash function
include linear perceptron functions (MLH, SPLH, LSH), kernel functions
(KSH, KLSH), eigenfunctions (SH, MDSH).
The optimization procedure is then coupled with the selected family of hash
function.
Different types of hash functions offer a trade-off between testing time and
ranking accuracy. For example, compared to kernel functions, the
simple linear perceptron function is usually much more efficient for
evaluation but can have a relatively low accuracy for nearest neighbour
search.
Moreover, this coupling often results in a highly
non-convex problem which is can be very difficult to optimize.

As an example, the loss functions in MDSH, KSH and BRE all take a
similar form that aims to minimize the difference between the hamming
affinity (or distance) and the ground truth of
data pairs. However, the optimization procedures used in these methods are
coupled with the form of hash functions (eigenfunctions, kernel
functions) and thus different optimization techniques are needed.

Self-Taught Hashing (STH)  \cite{zhang2010self} is a method which
decomposes the learning procedure into two steps: binary code
generating and hash function learning. We extend this idea and
proposed a general two-step approach to hashing
of which STH can be seen as a specific example.
Note that
STH optimizes the Laplacian affinity loss, which only tries to pull
together those similar data pairs but does
not push away those dissimilar data pairs
and that, as has been
shown in  manifold
learning, this may lead to inferior performance.

Our framework, however, is able to
accommodate many different loss functions defined on the hamming affinity
of data pairs, such as the loss function used in KSH, BRE or
MLH.
This more general family of loss functions may consider both
similar
and
dissimilar data pairs.  In order to produce effective binary codes in
this first step, we develop a new technique based on coordinate
descent. We show that at each iteration of coordinate descent,
we can formulate the
optimization problem of any hamming affinity loss as a  binary quadratic
problem (BQP).
This formulation unifies different types  of objective function into
the same optimization problem, which significantly simplifies the optimization
effort.
{Our main contributions are as follows.}

\begin{itemize}
\item[1.]
 We propose a flexible hashing framework that decomposes the learning
      procedure into two steps: binary codes inference step and hash
      function learning step. This decomposition simplifies the
      problem and enables the use of
      different types of loss functions and simplifies the hash function
      learning problem into a standard binary classification problem.
An arbitrary classifier, such as linear or kernel SVM,
      boosting, decision tree and neural networks, may thus be adopted
      to train the hash functions.

 \item[2.]
 For binary code inference, we show that optimization using
      different types of loss functions (e.g., loss functions in KSH,
      BRE, MLH) can be solved as a series of binary quadratic
      problems. We show that any type of loss function (e.g., the $
      \ell_2$ loss, exponential loss, hinge loss) defined on hamming
      affinity  of data pairs can be equivalently converted into
      a standard quadratic function. Based on this key observation, we
      propose a general block coordinate decent method that is able to
      incorporate  many different types of loss functions in a unified manner.

\item[3.]

      The proposed method is simple and easy to implement. We carry
      out extensive experiments on nearest neighbour search for image
      retrieval. To show the flexibility, we evaluate our method using several different types of loss
      functions and different formats of hash functions (linear SVM,
      kernel SVM, Adaboost with decision stumps, etc).
      Experiments show that our method outperforms the state-of-the-art.

\end{itemize}

\section{Two-step hashing}

Given a set of training points $\Xcal=\{\x_1, \x_2,...\x_n\} \subset
\Real^d$, the goal of hashing is to learn a set of hash functions that is able to
preserve some notion of similarity between data points. A ground truth
affinity (or distance) matrix, $\bY$, is provided (or calculated by a
pre-defined rule) for training, which
defines the (dis-)similarity relations between data pairs.
In this case $y_{ij}$ is the $ (i,j) $-th element of the matrix $\bY$, which is an
affinity value  of the data pair $(\x_i, \x_j)$.
As a simple example, if the data labels are available, $y_{ij}$ can be defined as $1$ for
 data pairs belonging to the same class and $-1$ for dissimilar data pairs. In the case of
unsupervised learning, $y_{ij}$ can be defined as the Euclidean
distance or Gaussian affinity on data points.
$\bh$ is a set of $m$ hash functions: $\bh=[h_1(\cdot), h_2(\cdot),
\dots, h_m]$.
The output of the hash functions are  $m$-bit binary codes: $\bh(\x) \in \{-1,1\}^m$.
In general, the optimization can be  written as:
 \begin{align}
 	\label{loss1}
	\min_{\bh} \sum_{i=1}^n\sum_{j=1}^n
    \delta_{ij} L(\bh(\x_i), \bh(\x_j); y_{ij}),
 \end{align}
where
$\delta_{ij} \in \{0, 1\}$  indicates
whether the relation between two data points is defined, and
 $L(\bh(\x_i), \bh(\x_j); y_{ij})$ is a loss function that
measures the how well the binary codes
match the expected affinity
(or distance) $y_{ij}$. Many different types
of loss function  $ L $ have been
devised, and will be discussed in detail in the next Section.

Most existing methods try to directly optimize objective
\eqref{loss1} in order to learn
the parameters of hash functions \cite{KSH,
norouzi2011minimal, kulis2009learning, MDSH}.
This inevitably means that the optimisation processes is tightly coupled to the
form of hash function used,
which makes it non-trivial to extend a method to use another different
format of hash function. Moreover, this coupling usually
results in highly non-convex NP-hard problems.
Following the idea of STH \cite{zhang2010self}, we decompose the
learning procedure into two steps: the first step for binary code
inference and the second step for hash function learning. The first
step is to solve the optimization:
 \begin{align}
 	\label{loss2}
	\min_{\Z} \sum_{i=1}^n\sum_{j=1}^n \delta_{ij}
    L(\z_i, \z_j; y_{ij}), \; \mbox{s.t.}\; \Z \in \{-1, 1\}^{n\times m},
 \end{align}
where $\Z$ is the matrix of $m$-bit binary codes for all data points, and
$\z_i$ is the binary code
row
vector corresponding to data point $i$.

The second step is to learn hash functions based on the binary codes
obtained in the first step, which is achieved by solving the optimization problem:
 \begin{align}
 	\label{loss3}
	\min_{\bh} \sum_{i=1}^n F(\z_i, \bh(\x_i)).
 \end{align}
 Here $F(\cdot,\cdot)$ is a loss function.
We solve the
above optimization independently for each of the $m$ bits. To learn the $k$-th
hash function ($h_k$), the optimization can be written:
 \begin{align}
 	\label{loss4}
	\min_{h_k} \sum_{i=1}^n
    F' (z_{i,k}, h_k(\x_i)).
 \end{align}
Here $F'(\cdot,\cdot)$ is an loss function defined on two codes.
Clearly, the above optimization is a binary
classification problem which is to minimize a kind of loss given
the binary labels.
For example it can be an
zero-one step function returning $1$
if two inputs have the same value, and $0$ otherwise.
As in classification, one can also use a convex surrogate to replace
the zero-one loss. Typical surrogate loss functions are hinge loss,
logistic loss, etc.
$z_{i,k}$ is the binary code corresponding to the $i$-th data point and the
$k$-th bit. The resulting classifier is the hash function that we
aim to learn. Therefore, we are able to use any form of classifier. For example,
we can learn perceptron hash functions by training a linear SVM. The
linear perceptron hash function has the form:
 \begin{align}
	h(\x) = \sign (\w^\T\x + b).
 \end{align}
We could also train, for example, an RBF-kernel SVM, or Adaboost with decision
trees for use as hash functions.  Here we describe a kernel
hash function that is learned using a linear SVM on kernel-transferred
features (referred to as SVM-KF). The hash function learned by SVM-KF
has a form as follows:
 \begin{align}
	h(\x) = \sign ( \textstyle \sum_{j=q}^Q w_q\kappa(\x_q', \x) + b),
 \end{align}
in which $\Xcal'=\{\x'_1,\dots,\x'_Q\}$ are $Q$ data points generated
from the training set by random or uniform sampling.

We evaluate variety of  different kinds of hash function in the Experiments
Section below.
These tests show that
 Kernel hash functions often offer better ranking precision
but require much more evaluation time than linear perceptron hash
functions. The hash functions learned by SVM-KF
represents a trade-off
between kernel SVM and linear SVM.

The method we propose is labelled Two-Step Hashing (\tsh),
the steps are as follows:

\noindent
    Step 1: Solving the optimization problem in \eqref{loss2} using block
    coordinate decent ({Algorithm \ref{alg}}) to obtain binary
    codes for each training data point.

\noindent
Step 2: Solving the binary classification problem in
    \eqref{loss4} for each bit based on the binary codes obtained at
    Step 1.

\begin{algorithm}[t!]
\caption{Block coordinate decent for learning binary codes (Step 1)}
\footnotesize{
	1: {\bf Input:} affinity matrix $\Y$, bit length $m$, number of cyclic iteration $r$.

	2: {\bf Initialize} the binary code matrix $\Z$.

	3: {\bf Repeat}

	4: \quad \quad \quad {\bf For} $t = 1, 2, \dots, m$

	5: \quad \quad \quad \quad \quad Solve the binary quadratic problem (BQP) in \eqref{loss6} for

\quad \quad \quad\quad \quad \quad \quad the $t$-th bit to obtain the binary code of $t$-th bit.

	6: \quad \quad \quad Update the code of $t$-th bit in code matrix $\Z$

	7: \quad\quad\quad {\bf End For}

    8: {\bf Until}
        the maximum cyclic iteration $r$ is reached.

    9: {\bf Output:} the matrix of binary codes $\Z$
    }
\label{alg}
\end{algorithm}

\section{Solving binary quadratic problems}

\label{sec:bqp}

Optimising \eqref{loss2} in Step 1 for the entire
binary code matrix can be difficult. Instead, we develop a block
coordinate descent method so that the problem at each iteration can be solved easily.
Moreover, we show that at each iteration, any pairwise hamming
affinity (or distance) based loss  can be equivalently  formulated as a binary quadratic
problem. Thus we are able to easily work with different loss functions.

Block coordinate decent (BCD) is a technique that iteratively
optimizes a subset of variables at a time.
For each iteration, we pick one bit for optimization in a cyclic
fashion.
The optimization for the $k$-th bit can be written as:
 \begin{align}
 	\label{loss41}
	\min_{\z_{(k)}}
    \sum_{i=1}^n\sum_{j=1}^n
    \delta_{ij} l_k(z_{i,k}, z_{j,k}), \;\;\; s.t.\; \z_{(k)} \in \{-1, 1\}^n,
 \end{align}
where $l_k$ is the loss function defined on the $k$-th bit:
$l_k(z_{i,k}, z_{j,k})=L(z_{i,k}, z_{j,k}, \bar{\z_{i}}, \bar{\z_{j}} ; y_{ij}).$
Here $\z_{(k)}$ contains the binary codes of the $k$-th bit. $z_{i,k}$ is
the binary code of the $i$-th data point and the $k$-th bit.
$\bar{\z_{i}}$ is the binary codes of the $i$-th data point excluding
the $k$-th bit.

Until now,  we have not described the form of the loss function $L$.
Our optimization method is not restricted to optimizing a specified form of
the loss function. Based on the following proposition, we are able to
rewrite any hamming affinity (or distance) based loss function $L$
into a standard quadratic problem.
\begin{proposition}
\label{pro:p1}
 For any loss function $l(z_1,z_2)$ that is defined on a pair of binary input variables $z_1,z_2 \in \{-1,1\}$ and $l(1,1)=l(-1,-1), l(1,-1)=l(1,-1)$,
we can define a quadratic function $g(z_1,z_2)$ that is equal to $l(z_1,z_2)$. %
We have following equation:
 \begin{align}
	l(z_1,z_2)& = \frac{1}{2} \biggr[ z_1 z_2 (l^{(11)} - l^{(-11)}) + l^{(11)} + l^{(-11)} \biggr], \\
	& = \frac{1}{2} z_1 z_2 (l^{(11)} - l^{(-11)}) + const\\
	& = g(z_1,z_2).
  \end{align}
Here $l^{(11)}, l^{(-11)}$ are constants, $l^{(11)}$ is the loss
output on identical input pair: $l^{(11)}=l(1,1)$, and $l^{(-11)}$ is
the loss output on distinct input pair: $l^{(-11)}=l(-1,1)$.
\end{proposition}

\begin{proof}
    This proposition can be easily proved by exhaustively checking all
    possible inputs of the loss function. Notice that there are only
    two possible output values of the loss function.
For the input $(z_1=1,z_2=1)$:
 \begin{align}
	g(1,1) & = \frac{1}{2} \biggr[ 1 \times 1 \times  (l^{(11)} - l^{(-11)}) + l^{(11)} + l^{(-11)} \biggr] \notag \\
	& = l(1,1), \notag
  \end{align}
For the input $(z_1=-1,z_2=1)$:
 \begin{align}
	g(-1,1) & = \frac{1}{2} \biggr[ -1 \times 1 \times (l^{(11)} - l^{(-11)}) + l^{(11)} + l^{(-11)} \biggr] \notag \\
	& = l(-1,1), \notag
  \end{align}
The input $(z_1=-1,z_2=-1)$ is the same as $(z_1=1,z_2=1)$ and the input $(z_1=1,z_2=-1)$ is the same as $(z_1=-1,z_2=1)$.
In conclusion, the function $l$ and $g$ have the same output for any possible inputs.
\end{proof}

Any hash loss function $l$ which is defined on the hamming affinity between, or
hamming distance of, data pairs is able to meet the
requirement that: $l(1,1)=l(-1,-1), l(1,-1)=l(1,-1)$. Applying this
proposition, the optimization of \eqref{loss41} can be equivalently
reformulated as:
 \begin{align}
 	\label{loss5}
	\min_{\z_{(k)} \in \{-1, 1\}^n} \sum_{i=1}^n\sum_{j=1}^n \delta_{ij} (l_{k,i,j}^{(11)} - l_{k,i,j}^{(-11)}) z_{i,k}z_{j,k},
 \end{align}
The above optimization is an unconstrained binary quadratic problem.
Let $a_{i,j}$ denote the $ (i,j) $-th
element of matrix $\A$,
which we define as:
 \begin{align}
a_{i,j}=\delta_{ij} (l_{k,i,j}^{(11)} - l_{k,i,j}^{(-11)}).
 \end{align}
The above optimization \eqref{loss5} can be written in matrix form:
 \begin{align}
 	\label{loss6}
	\min_{\z_{(k)}} \z_{(k)}^\T \A \z_{(k)}
    \;\;\; {\rm s.t.}\; \z_{(k)} \in \{-1, 1\}^n.
 \end{align}
We have shown that at each iteration, the original optimization in
\eqref{loss41} can be equivalently  reformulated as a binary quadratic
problem (BQP) in \eqref{loss6}. BQP has been extensively
studied. To solve \eqref{loss6},  we first apply the spectral
relaxation to get an initial solution.
Spectral relaxation drops the binary constraints. The optimization becomes
 \begin{align}
 	\label{loss7}
	\min_{\z_{(k)}} \z_{(k)}^\T \A \z_{(k)}, \;\;\;
    {\rm s.t.} \; \|\z_{(k)}\|_2^2=n
 \end{align}
The solution (denoted $\z_{(k)}^0$) of the above optimization is
simply the eigenvector that corresponds to the minimum eigenvalue of
the matrix $\A$.
To achieve a better solution,
here we take a step further. We solve the following relaxed problem of
\eqref{loss6} as follows
 \begin{align}
 	\label{loss8}
	\min_{\z_{(k)}} \z_{(k)}^\T \A \z_{(k)}, \;\;\;
    {\rm s.t.} \; \z_{(k)} \in =[-1, 1]^n.
 \end{align}
This relaxation is tighter than the spectral relaxation and provides
a solution of better quality.
To solve the above problem, we use the solution $\z_{(k)}^0$ of
spectral relaxation in \eqref{loss7} as initialization and solve it
using the efficient LBFGS-B solver \cite{lbfgs}.
An alternative is to
directly solve the above optimization \eqref{loss8} using random
initialization, without solving the spectral relaxation.  The
algorithm of our block coordinate decent for binary code inference in
S1 is summarized in  Algorithm \ref{alg}.

The approach proposed above is applicable to
many different types of loss functions,
which are
defined on hamming distance or hamming affinity, such as the $\ell_2$
loss, exponential loss, hinge loss.
Here we describe
a selection of such loss functions,
most of which arise from recently
proposed hashing methods. We evaluate these loss functions in the
Experiments Section below. Note that $m$ is the number of bits, and $d_h(\cdot,\cdot)$
is the hamming distance on data pairs. If not specified, $y_{ij}=1$ if
the data pair is similar, and $y_{ij}=-1$ if the data pair is
dissimilar. $\delta(\cdot) \in \{0, 1\}$ is an indicator function.

{\bf TSH-KSH} The KSH loss function is based on hamming affinity using $\ell_2$ loss function. MDSH also uses a similar form of loss function (weighted hamming affinity instead).
 \begin{align}
	L_{\mathrm{KSH}}(\z_i, \z_j)= (\z_i^\T\z_j -my_{ij})^2
 \end{align}

{\bf TSH-BRE} The BRE loss function is based on hamming hamming distance using $\ell_2$ loss function. Here the definition of $y_{ij}$ is different: $y_{ij}=0$ if the data pair is similar, $y_{ij}=1$ if the data pair is dissimilar.
 \begin{align}
	L_{\mathrm{BRE}}(\z_i, \z_j)= (d_h(\z_i,\z_j)/m - y_{ij})^2
 \end{align}

{\bf TSH-SPLH} Uses an exponential loss outside the loss function proposed in SPLH which is based on the hamming affinity of data pairs.
 \begin{align}
	L_{\mathrm{SPLH}}(\z_i, \z_j)= \exp(\frac{-y_{ij}\z_i^\T\z_j}{m})
 \end{align}

{\bf TSH-EE} Elastic Embedding (EE) is a dimension reduction method proposed in \cite{EE}. Here we use their loss function with some modifications, which is a exponential based on distance. $\lambda$ here is a trade-off parameter.
  \begin{align}
	L_{\mathrm{EE}}(\z_i, \z_j) & = \delta({y_{ij}>0})d_h(\z_i, \z_j) \notag\\
		&+ \lambda \delta({y_{ij}<0}) \exp [ -d_h(\z_i, \z_j)/m].
 \end{align}

{\bf TSH-ExpH} Here ExpH is an exponential loss function using the  hamming distance:
 \begin{align}
	L_{\mathrm{ExpH}}(\z_i, \z_j)=
    \exp\biggr[\frac{y_{ij}d_h(\z_i,\z_j) +
    m\delta(y_{ij}<0)}{m}\biggr].
 \end{align}

\begin{table*}
\caption{Results (using
hash codes of 32 bits)
 of TSH using different
loss functions, and a selection of other supervised and unsupervised methods on 3
datasets. The upper part relates the results on training data and the lower
on testing data.
The results show that Step 1  of our method is able to generate
effective binary codes that
outperforms those of competing methods
 on the training
data. On the testing data our method also outperforms others by
a large margin.}
\centering
\resizebox{0.985\linewidth}{!}
  {
  \begin{tabular}{ r | c c c| c c c| c c c}
  \hline \hline

	& \multicolumn{3}{|c}{Precision-Recall}	& \multicolumn{3}{|c}{MAP}	& \multicolumn{3}{|c}{ Precision at K (K=300)}\\
\hline
Datasets	&LABELME	&MNIST	&CIFAR10	&LABELME	&MNIST	&CIFAR10	&LABELME	&MNIST	&CIFAR10\\
\hline
&	\multicolumn{9}{|c}{Results on training data}\\
\hline
TSH-KSH	&0.501	&1.000	&1.000	&0.570	&1.000	&1.000	&0.229	&0.667	&0.667\\
TSH-BRE	&0.527	&1.000	&1.000	&0.600	&1.000	&1.000	&0.230	&0.667	&0.667\\
TSH-SPLH	&0.504	&1.000	&1.000	&0.524	&1.000	&1.000	&0.230	&0.667	&0.667\\
TSH-EE	&0.485	&1.000	&1.000	&0.524	&1.000	&1.000	&0.224	&0.667	&0.667\\
TSH-ExpH	&0.475	&1.000	&1.000	&0.541	&1.000	&1.000	&0.225	&0.667	&0.667\\
\hline
STHs	&0.335	&0.800	&0.629	&0.387	&0.882	&0.774	&0.176	&0.575	&0.433\\
KSH	&0.283	&0.892	&0.585	&0.316	&0.967	&0.652	&0.168	&0.647	&0.481\\
BREs	&0.161	&0.445	&0.220	&0.153	&0.504	&0.190	&0.097	&0.376	&0.171\\
SPLH	&0.166	&0.500	&0.292	&0.153	&0.588	&0.302	&0.092	&0.422	&0.260\\
MLH	&0.120	&0.547	&0.190	&0.142	&0.685	&0.235	&0.100	&0.478	&0.200\\
\hline
\hline
&	\multicolumn{9}{|c}{Results on testing data}\\
\hline
TSH-KSH	&\bf 0.175	&0.843	&0.282	&\bf 0.296	&0.893	&0.440	&\bf 0.293	&0.889	&0.410\\
TSH-BRE	&0.169	&\bf 0.844	&0.283	&0.293	&0.896	&0.439	&\bf 0.293	&0.890	&0.409\\
TSH-SPLH	&0.174	&0.840	&\bf 0.284	&0.291	&0.895	&\bf 0.444	&0.288	&0.891	&\bf 0.416\\
TSH-EE	&0.169	&0.843	&0.280	&0.288	& 0.896	&0.438	&0.286	&0.892	&0.410\\
TSH-ExpH	&0.172	&\bf 0.844	&0.282	&0.287	&0.892	&0.441	&0.286	&0.887	&0.410\\
\hline
STHs	&0.094	&0.385	&0.144	&0.162	&0.639	&0.229	&0.156	&0.634	&0.218\\
STHs-RBF	&0.151	&0.674	&0.178	&0.274	&\bf 0.897	&0.354	&0.271	&\bf  0.893	&0.352\\
KSH	&0.165	&0.781	&0.249	&0.279	&0.884	&0.407	&0.158	&0.881	&0.398\\
BREs	&0.106	&0.409	&0.151	&0.178	&0.703	&0.226	&0.171	&0.702	&0.210\\
MLH	&0.100	&0.470	&0.150	&0.181	&0.648	&0.264	&0.174	&0.623	&0.215\\
SPLH	&0.093	&0.452	&0.191	&0.168	&0.714	&0.321	&0.158	&0.708	&0.315\\
ITQ-CCA	&0.077	&0.619	&0.206	&0.143	&0.792	&0.333	&0.133	&0.784	&0.325\\
\hline
MDSH	&0.100	&0.298	&0.150	&0.178	&0.691	&0.288	&0.155	&0.685	&0.228\\
SHPER	&0.102	&0.296	&0.152	&0.185	&0.624	&0.244	&0.176	&0.623	&0.233\\
ITQ	&0.116	&0.386	&0.161	&0.206	&0.750	&0.264	&0.197	&0.751	&0.252\\
AGH	&0.096	&0.404	&0.144	&0.194	&0.743	&0.252	&0.187	&0.744	&0.244\\
STH	&0.077	&0.361	&0.135	&0.135	&0.593	&0.216	&0.125	&0.644	&0.204\\
BRE	&0.091	&0.323	&0.137	&0.160	&0.651	&0.238	&0.147	&0.582	&0.185\\
LSH	&0.069	&0.211	&0.123	&0.116	&0.459	&0.188	&0.103	&0.448	&0.162\\

\hline
  \end{tabular}
  }

\label{tab:rank}
\end{table*}

\begin{table}
\caption{Training time (in seconds) for  TSH using
different loss functions, and several other supervised methods on 3
datasets. The value inside a brackets is the time used in the first step for
inferring the binary codes. The results show that our method is efficient.
Note that the second step of learning the hash functions can be easily parallelised.
 }
\centering
  {
  \begin{tabular}{ r | c c c c }
  \hline \hline
	&LABELME	&MNIST	&CIFAR10\\ \hline
TSH-KSH	&198 (107)	&341 (294)	&326 (262)\\
TSH-BRE	&133 (33)	&309 (264)	&234 (175)\\
TSH-EE	&124 (29)	&302 (249)	&287 (225)\\
TSH-ExpH	&128 (43)	&334 (281)	&344 (256)\\
\hline
STHs-RBF	&133	&99	&95\\
KSH	&326	&355	&379\\
BREs	&216	&615	&231\\
MLH	&670	&805	&658\\
\hline
  \end{tabular}
  }

\label{tab:trn_time}
\end{table}

\section{Experiments}

We compare with a few state-of-the-art hashing methods, including 6
(semi-)supervised methods: Supervised Hashing with Kernels (KSH)
\cite{KSH}, Iterative Quantization with supervised embedding (ITQ-CCA)
\cite{gong2012iterative}, Minimal Loss Hashing (MLH)
\cite{norouzi2011minimal}, Supervised Binary Reconstructive Embeddings
(BREs) \cite{kulis2009learning} and its unsupervised version BRE,
Supervised Self-Taught Hashing (STHs) \cite{zhang2010self} and its
unsupervised version STH, Semi-supervised sequential Projection
Learning Hashing(SPLH) \cite{wang2010semi}, and 7 unsupervised
methods: Locality-Sensitive Hashing (LSH) \cite{Gionis1999}, Iterative
Quantization (ITQ) \cite{gong2012iterative}, Anchor Graph Hashing
(AGH) \cite{liu2011hashingGraphs}, Spectral Hashing (SPH
\cite{weiss2008spectral}), Spherical Hashing (SPHER) \cite{jae2012},
Multi-dimension Spectral Hashing (MDSH) \cite{MDSH}, Kernelized
Locality-Sensitive Hashing KLSH \cite{KLSH}.  For comparison methods,
we  follow the original papers for parameter
setting. For SPLH, the regularization trade-off parameter is picked
from $0.01$ to $1$.  We use the hierarchical variant of AGH. For each
dataset, the bandwidth parameters of Gaussian affinity in MDSH and RBF
kernel in KLSH, KSH and our method \tsh is set as $\sigma=t\bar{d}$.
Here $\bar{d}$ is the average Euclidean distance of top $100$ nearing
neighbours and $t$ is picked from $0.01$ to $50$. For STHs and our
method \tsh, the trade-off parameter in SVM is picked from $10/n$ to
$10^{5}/n$, $n$ is the number of data points. For our \tsh-EE using EE
lost function, we simply set the trade-off parameter $\lambda$ to 100.
If not specified, our method \tsh use SVM with RBF kernel as hash
functions. The cyclic iteration number $r$ in our method is simply set to $1$.

We use 2 large scale image datasets and another 3 datasets for
evaluation. 2 large image datasets are $580,000$ tiny image dataset
(Tiny-580K) \cite{gong2012iterative}, and Flickr 1 Million image
dataset. %
Another 3 datasets include
CIFAR10, %
MNIST and LabelMe \cite{norouzi2011minimal}.
For the LabelMe dataset, the ground truth pairwise affinity matrix is provided.
For other datasets,
we use the multi-class labels to define the ground truth affinity by label agreement.

Tiny-580K is used in
\cite{gong2012iterative}.
Flick-1M dataset consists of 1 million thumbnail images
of the MIRFlickr-1M
We generate 320-dimension GIST features.

For these 2 large datasets, there is no semantic ground truth affinity
provided.  Following the same setting as other hash methods
\cite{wang2010semi, KSH}, we generate pseudo-labels for supervised
methods according to the $\ell_2$ distance. In detail, a data point
is labelled as a relevant neighbour to the query if it lies in the top
2 percentile points of the whole database.  For all datasets,
following a common setting in many supervised hashing methods
\cite{norouzi2011minimal, kulis2009learning, KSH}, we randomly select
2000 examples as testing queries, and the rest is served as database.
We use 2000 examples in the database for training. We use 4 types of
evaluation measures: Precision-at-K, Mean Average Precision (MAP),
Precision-Recall, Precision within Hamming distance 2.

\begin{figure}
    \centering

        \figcenter{\includegraphics[height=0.403in]{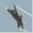}}
        \colseperator
        \figcenter{\includegraphics[height=1.0956in]{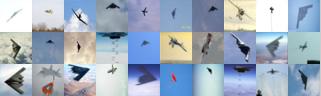}}

        \figcenter{\includegraphics[height=0.403in]{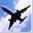}}
        \colseperator
        \figcenter{\includegraphics[height=1.0956in]{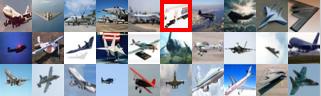}}

        \figcenter{\includegraphics[height=0.403in]{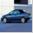}}
        \colseperator
        \figcenter{\includegraphics[height=1.0956in]{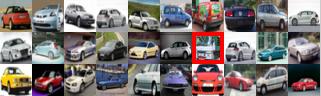}}

    \caption{Some retrieval examples of our method \tsh on CIFAR10. The first column shows query images, and the rest are top 30 retrieved images in the database. False predictions are marked by red boxes.
    }
   	\label{fig:examples}
\end{figure}

\subsection{Using different loss functions}

We evaluate the performance of our method \tsh using different loss functions on 3 datasets: LabelMe, MNIST, CIFAR10.
3 types of evaluation measures are used here: Precision-at-K, Mean
Average Precision (MAP) and the area under the Precision-Recall curve.
The loss function is defined in Section \ref{sec:bqp}. In particular,
our method
TSH-KSH uses the KSH \cite{KSH} loss function, TSH-BRE uses the BRE
\cite{kulis2009learning} function.
STHs-RBF is the STHs method using RBF kernel hash functions.
Our method also uses SVM with RBF kernel as hash functions.

First,  we evaluate the effectiveness of the Step-1 in our method.
We compare the quality of the generated binary codes on the {\em training
data points}.  The results are shown in the upper part of the table in
Table.\ \ref{tab:rank}.
The results show that our methods  generate high-quality binary
codes and outperform others by a large margin. In CIFAR10 and MNIST,
we are able to generate perfect codes that match the ground truth
similarity.  This demonstrates the effectiveness of coordinate descent
based hashing codes learning procedure (Step 1 of our framework).

Compared to STHs-RBF, even though we are using the same formate of hash function,
our overall objective function may be more effective and the
algorithm for code inference is more efficient. Thus our method
achieves better performance than STH.

The second part of the result in Table. \ref{tab:rank} shows the
testing performance. Our method also outperforms others in most cases.
Note that MNIST is a `simple' dataset and not as challenging as CIFAR10
and LabelMe. Thus many methods manage to achieve good performance.
In the challenging dataset CIFAR10 and LabelMe, our method
significantly outperforms others by a large margin.

Overall, %
for preserving the semantic similarity, supervised methods usually
perform much better than those unsupervised methods, which is
expected.
Our method performs
the best, and the running-up methods are STHs-RBF, KSH, and ITQ-CCA.

We show further results of using different numbers of bits in
Fig.\ \ref{fig:sup} and Fig.\ \ref{fig:unsup}
on the dataset CIFAR10 and LabelMe.
Our method still performs the best in most cases.
Some search examples are shown in Fig. \ref{fig:examples}.

\begin{figure*}
    \centering

   \includegraphics[width=.28\linewidth]{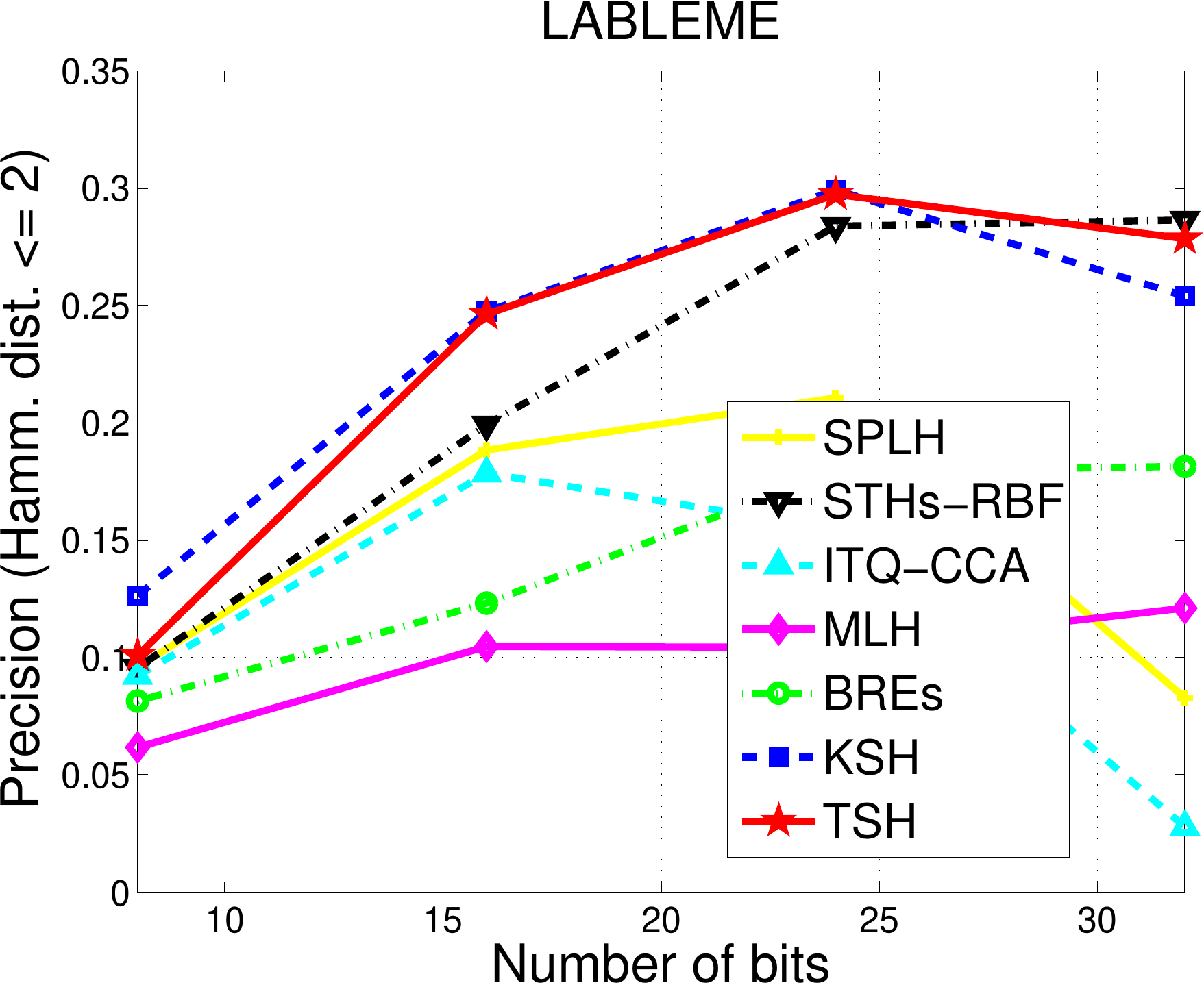}
   \includegraphics[width=.28\linewidth]{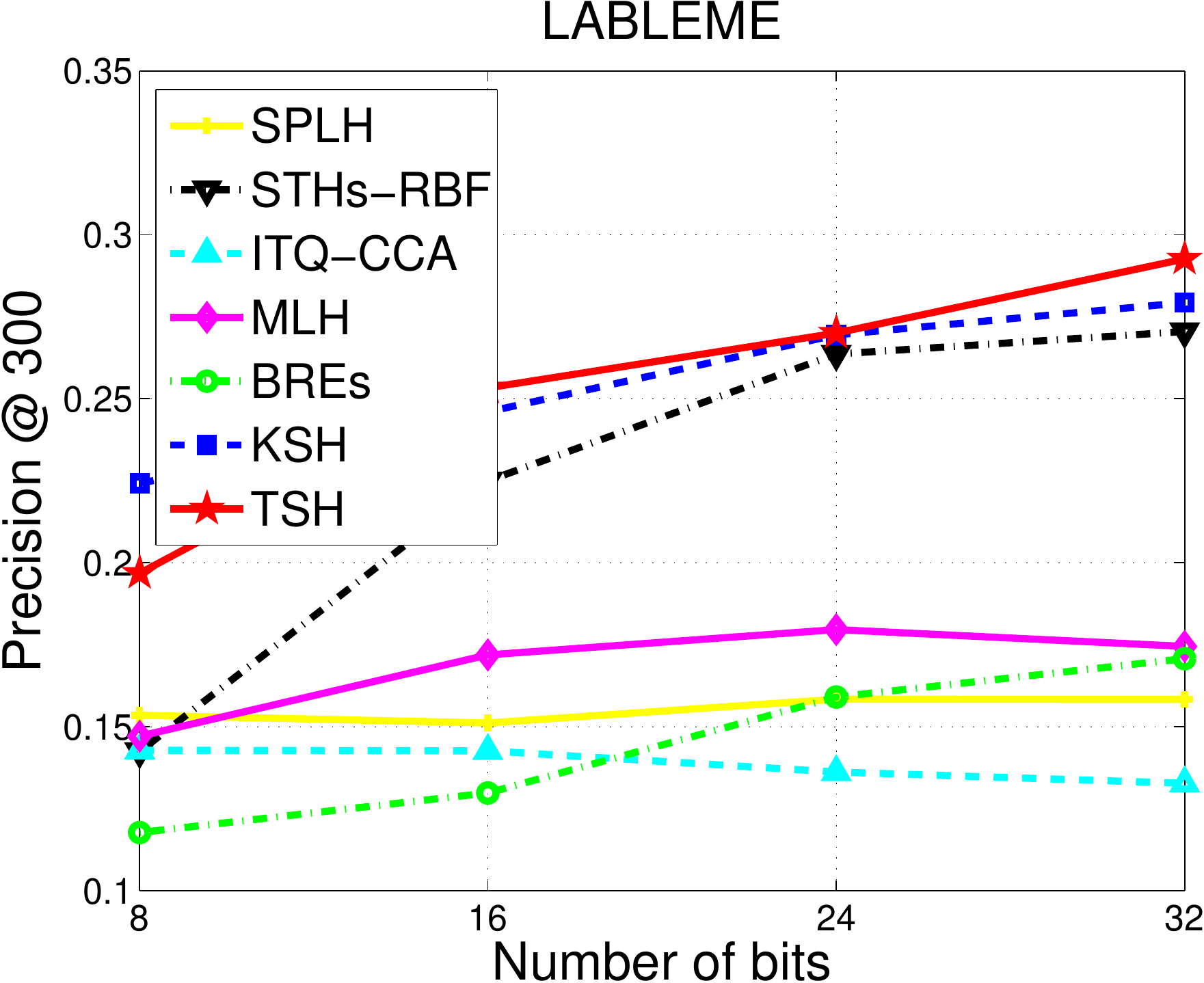}
   \includegraphics[width=.28\linewidth]{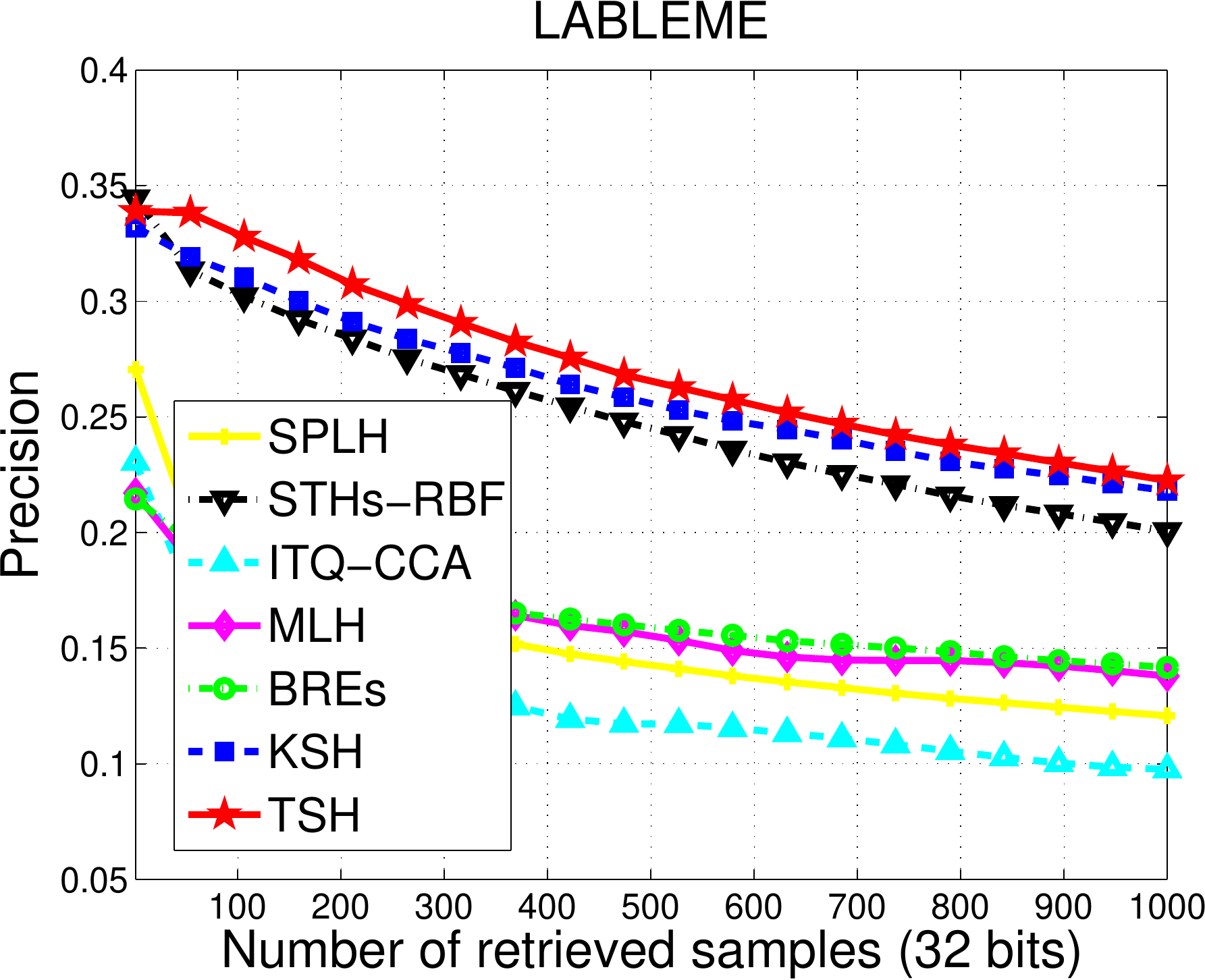}

   \includegraphics[width=.28\linewidth]{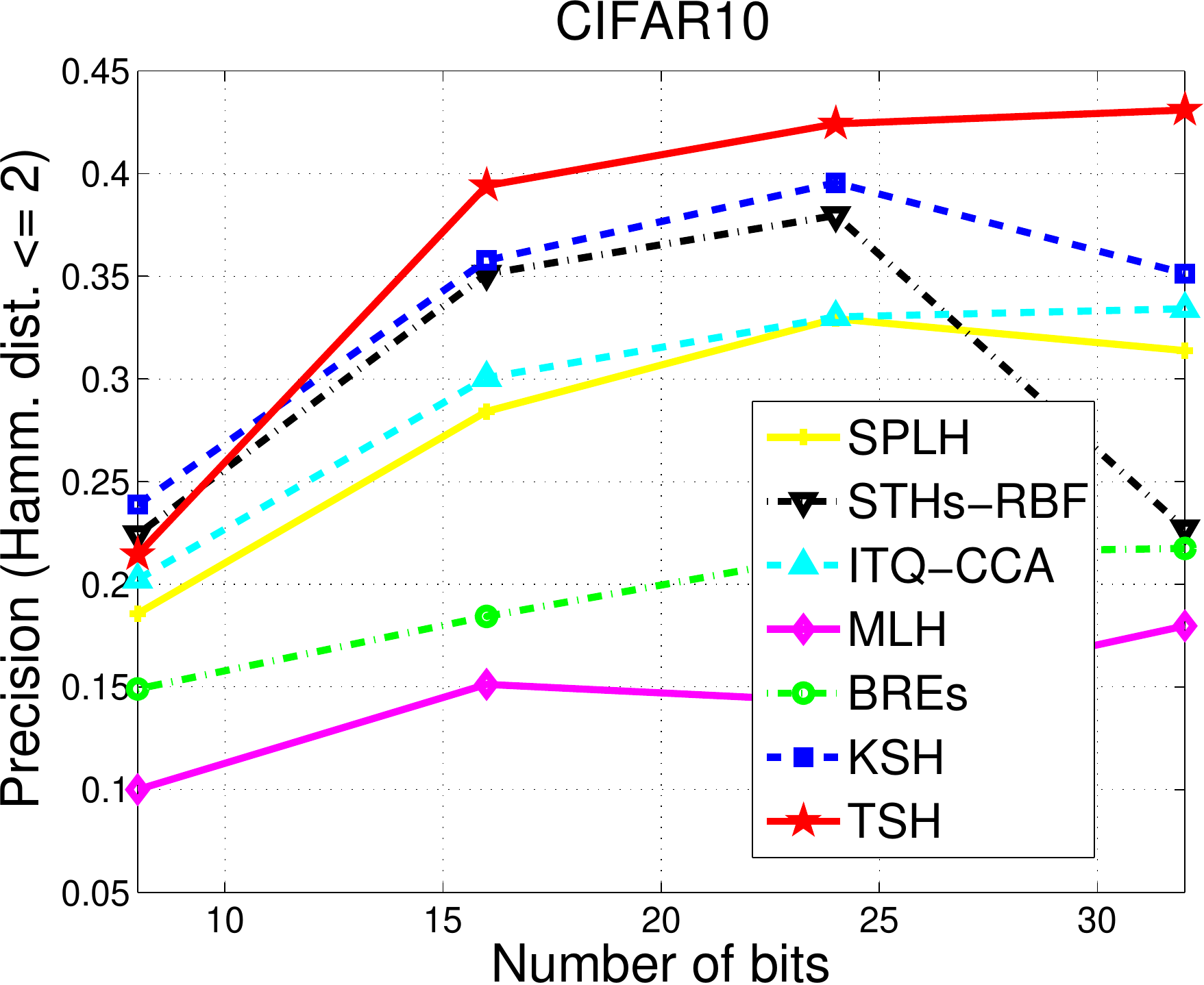}
   \includegraphics[width=.28\linewidth]{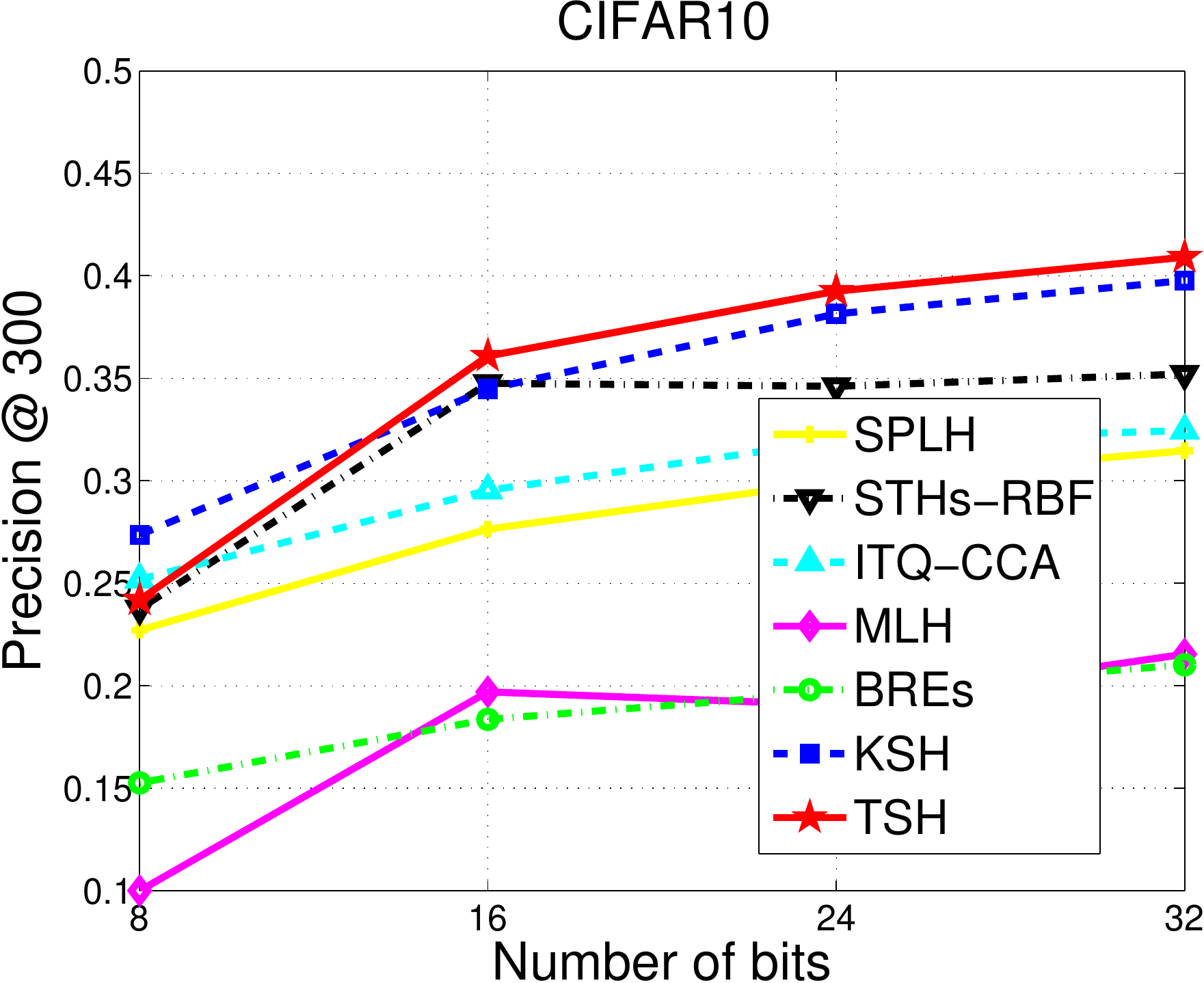}
   \includegraphics[width=.28\linewidth]{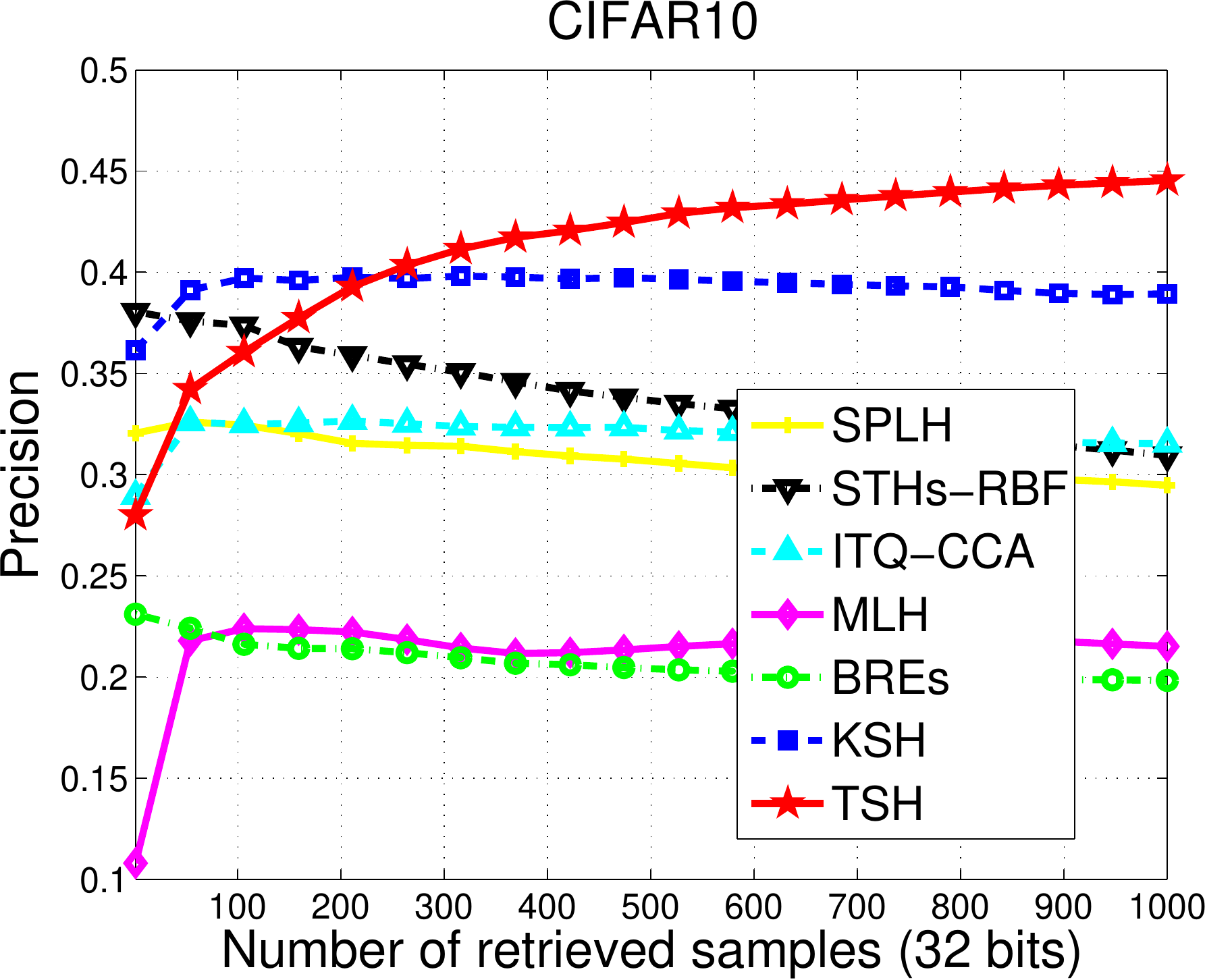}

      \caption{Results on 2 datasets of supervised methods. TSH denotes our method using BRE loss function. Results show that TSH outperforms others usually by a large margin. The runner-up methods are STHs-RBF and KSH.}
   	\label{fig:sup}
\end{figure*}

\begin{figure*}
    \centering

   \includegraphics[width=.28\linewidth]{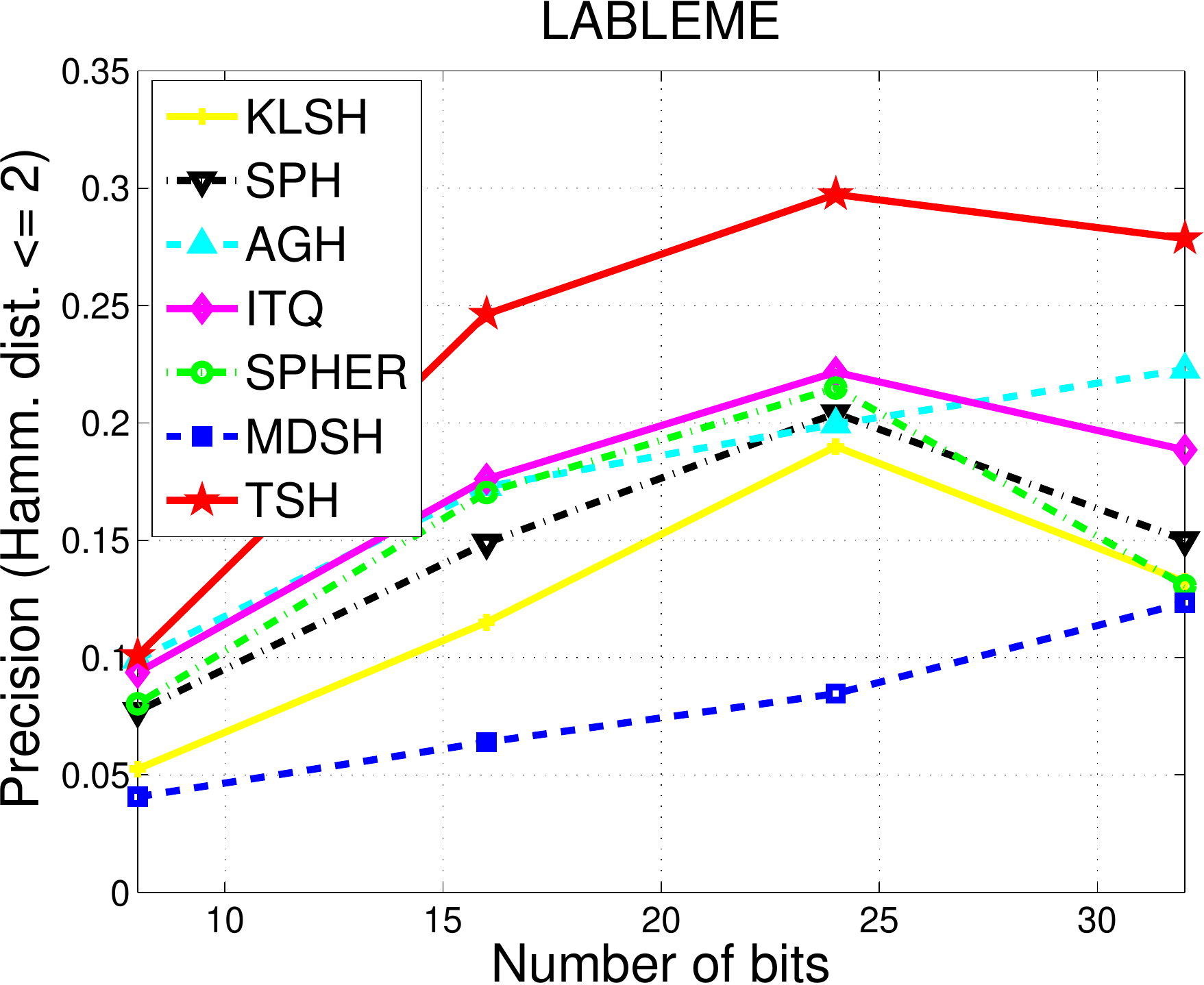}
   \includegraphics[width=.28\linewidth]{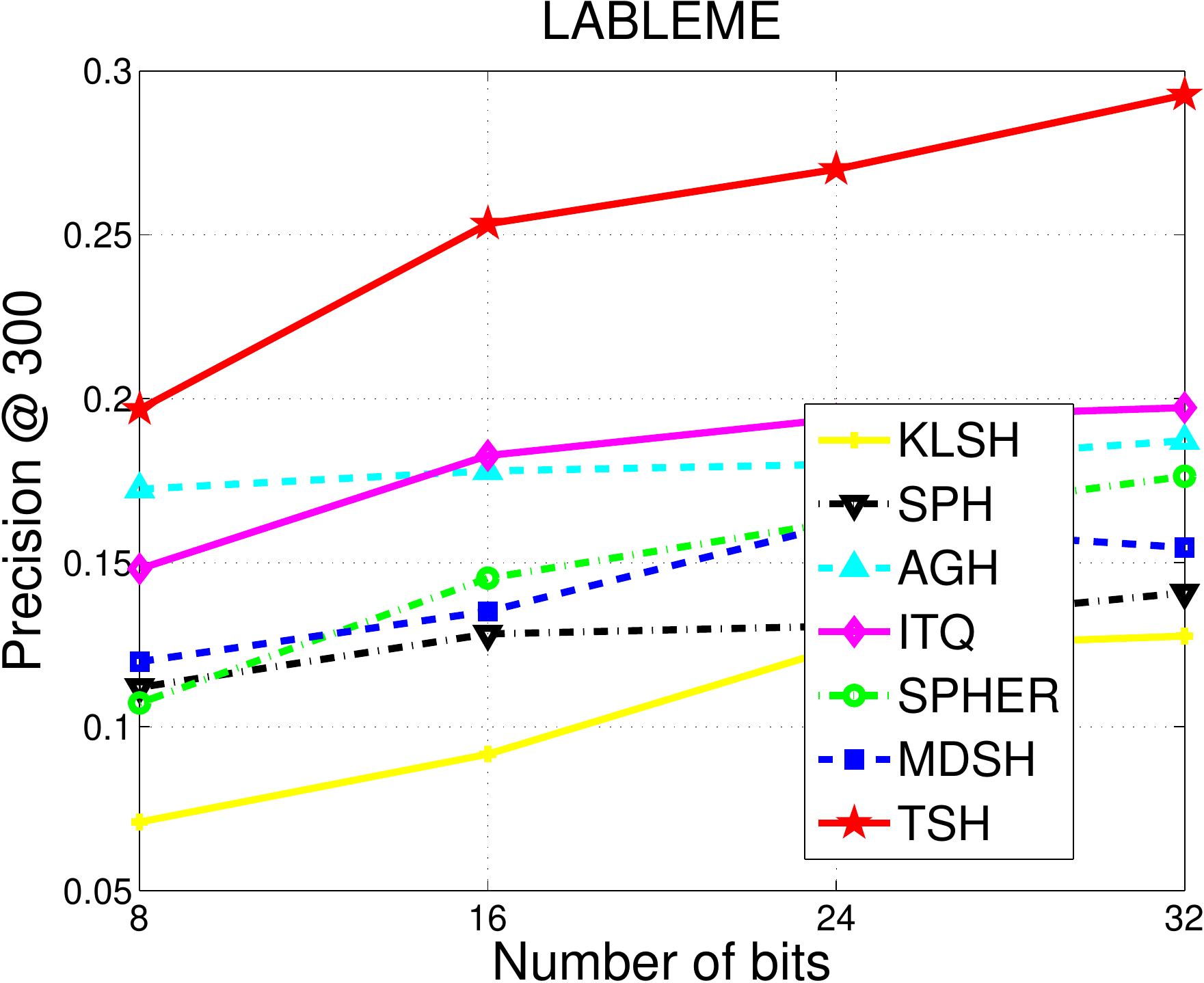}
   \includegraphics[width=.28\linewidth]{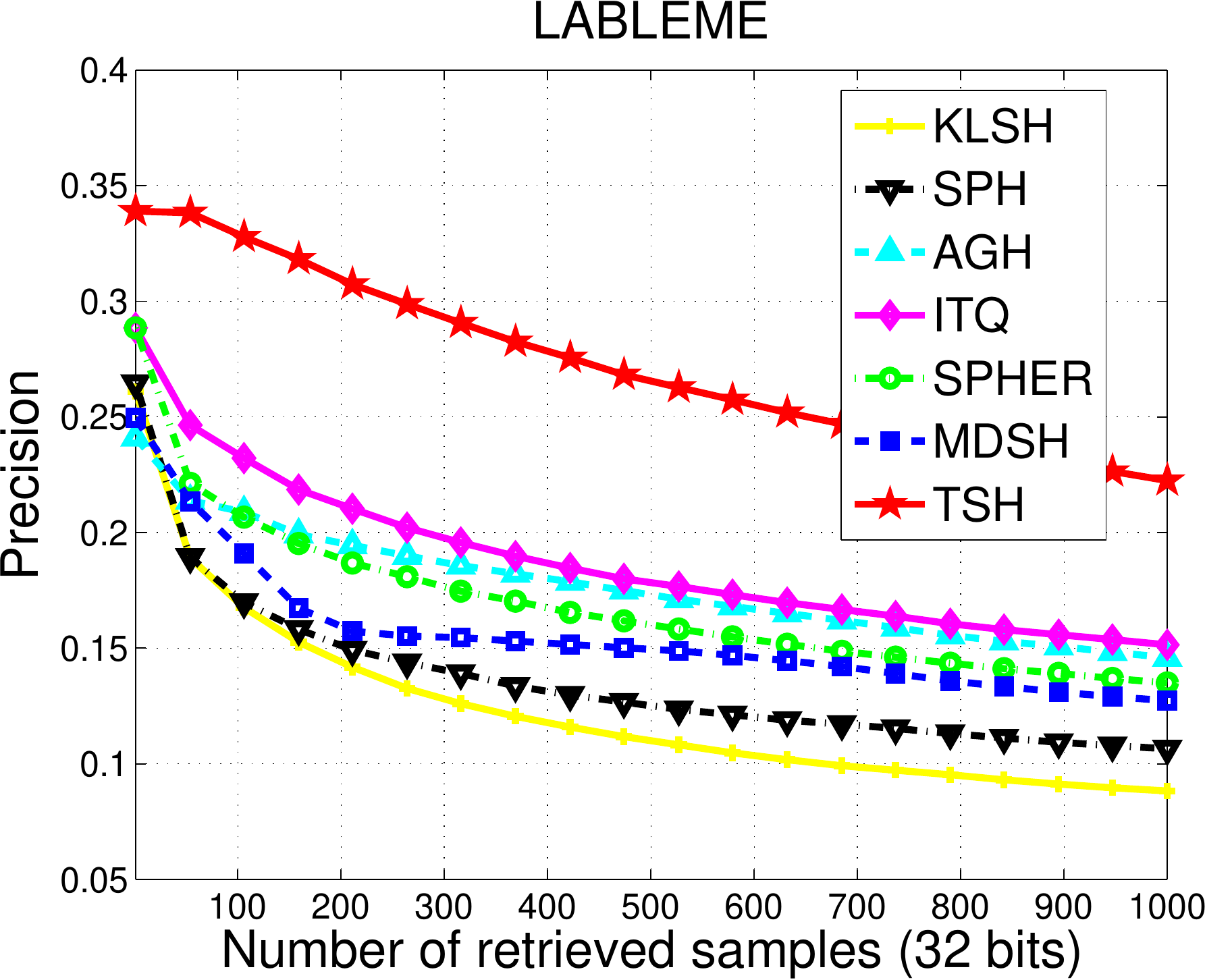}

   \includegraphics[width=.28\linewidth]{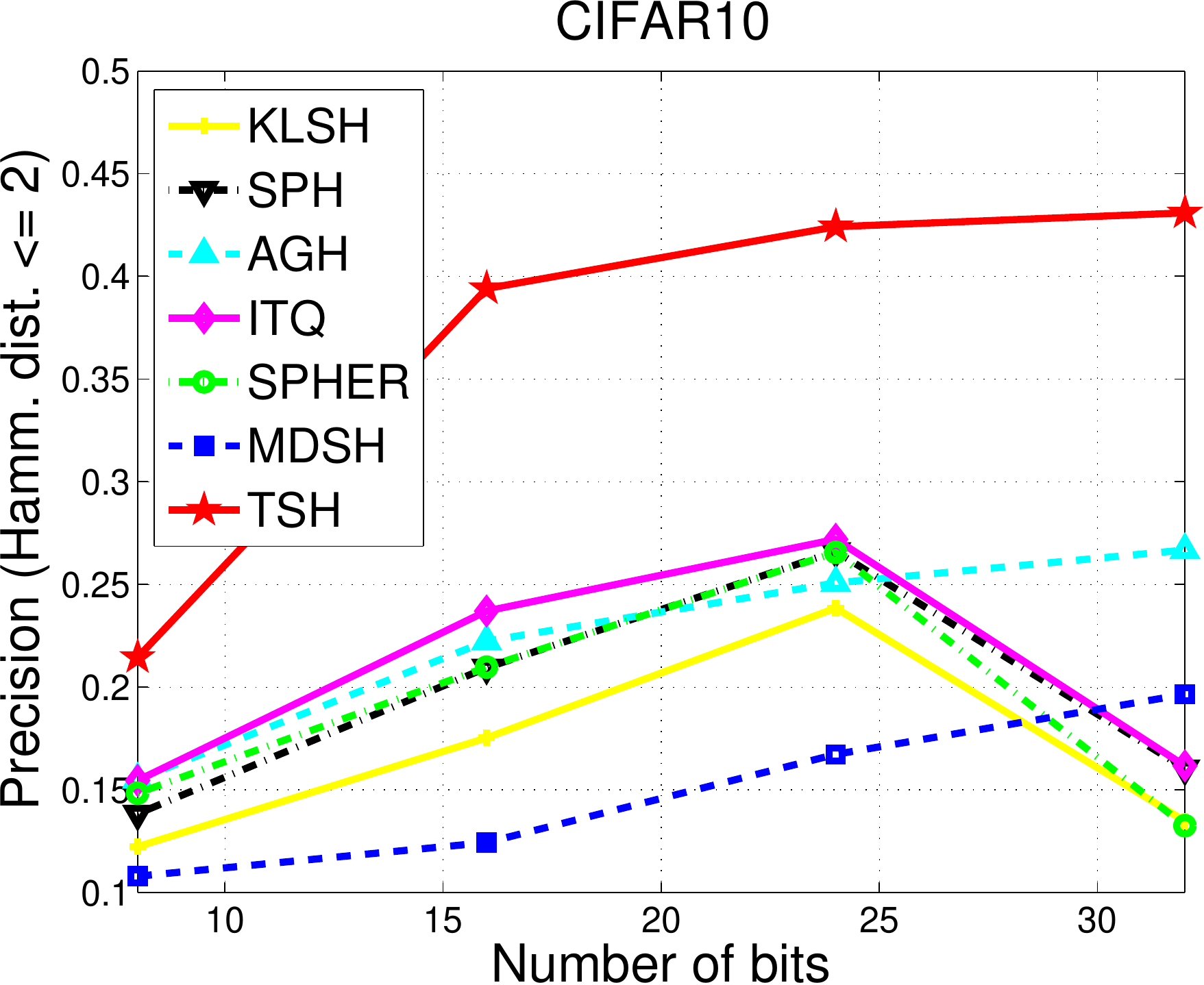}
   \includegraphics[width=.28\linewidth]{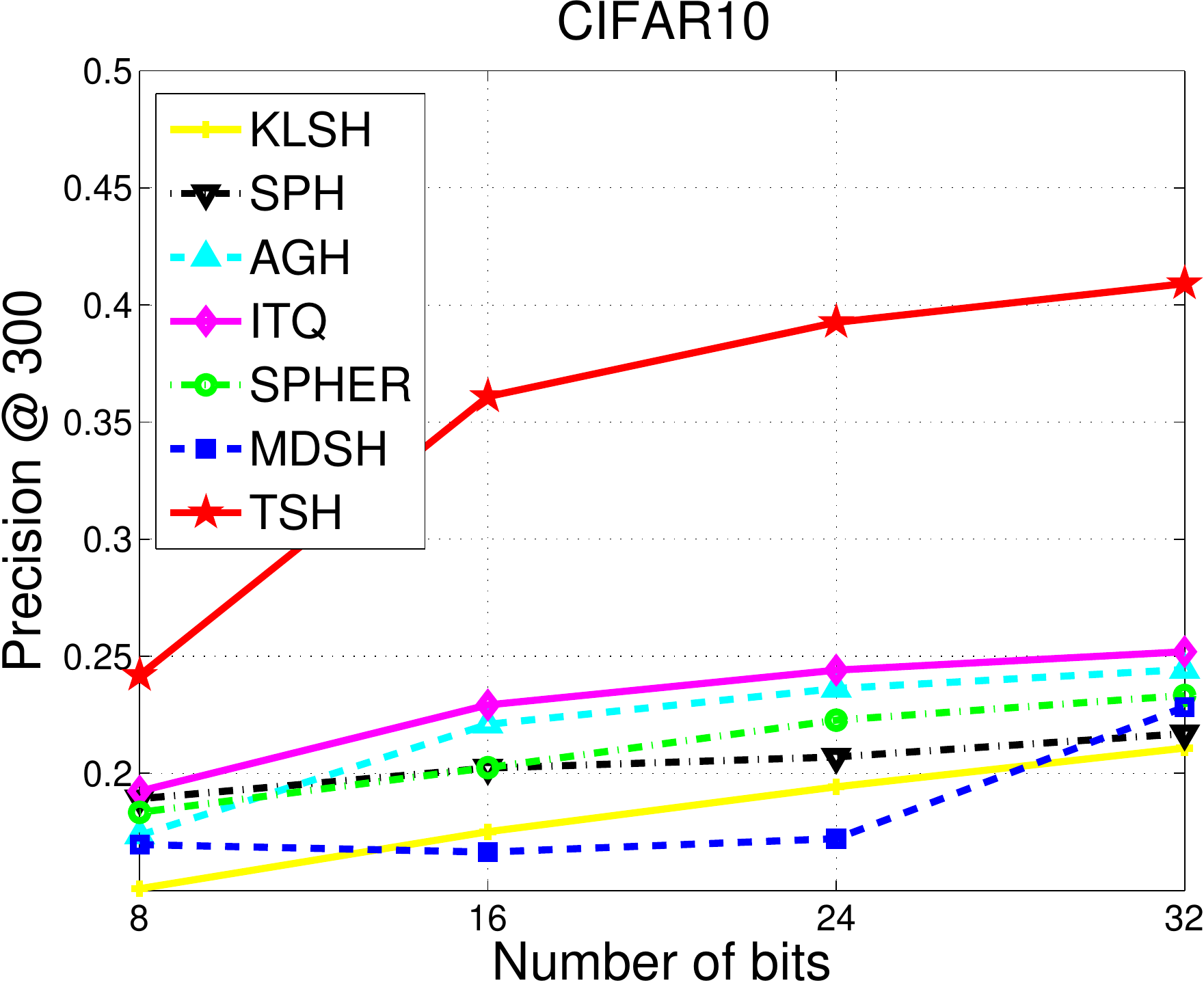}
   \includegraphics[width=.28\linewidth]{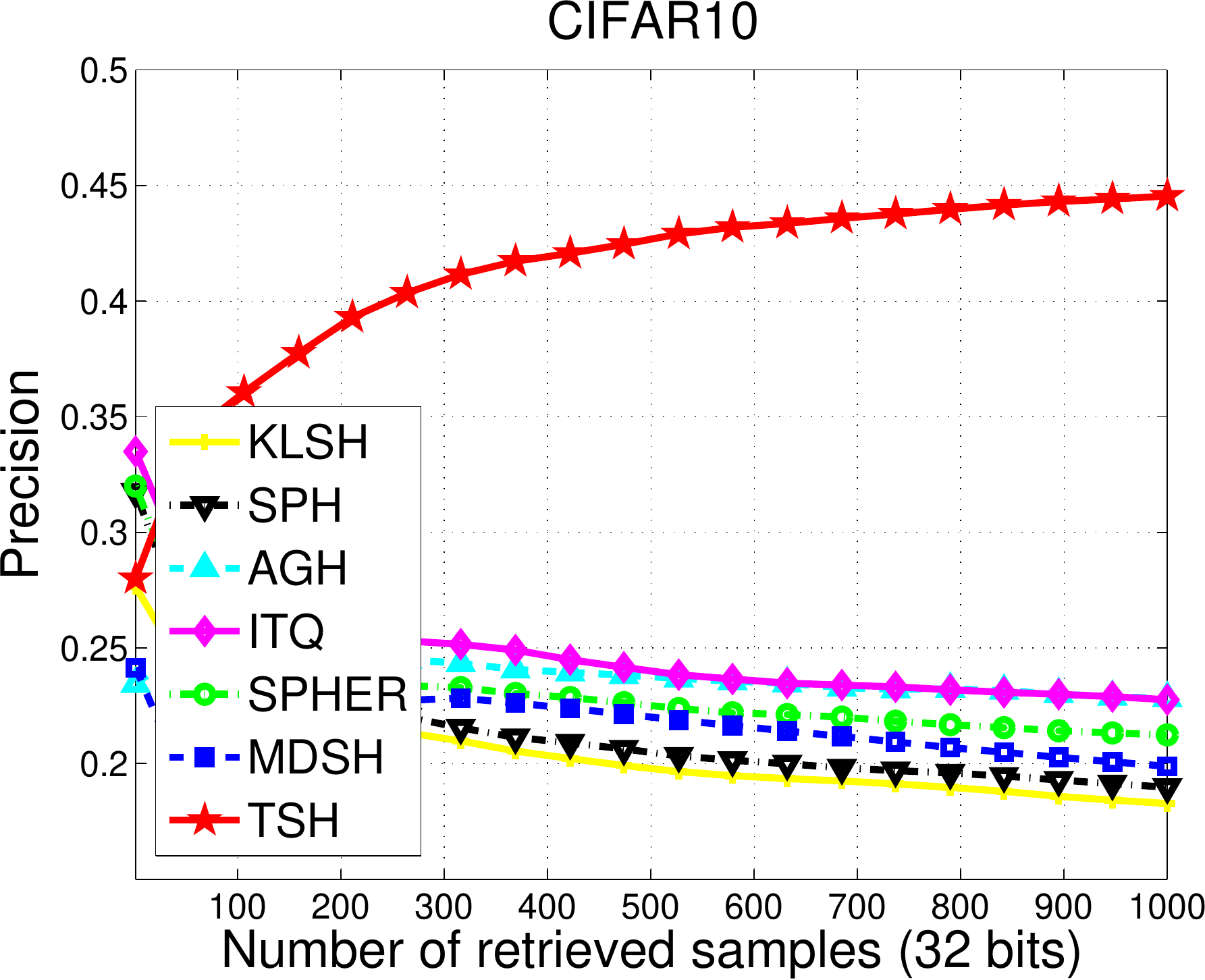}

      \caption{Results on 2 datasets for
      comparing  unsupervised methods.
      TSH denotes our method using the BRE loss function. Results show
      that TSH outperforms others usually by a large margin. The
      running-up methods are STHs-RBF and KSH.}
   	\label{fig:unsup}
\end{figure*}

\begin{figure*}
    \centering

   \includegraphics[width=.28\linewidth]{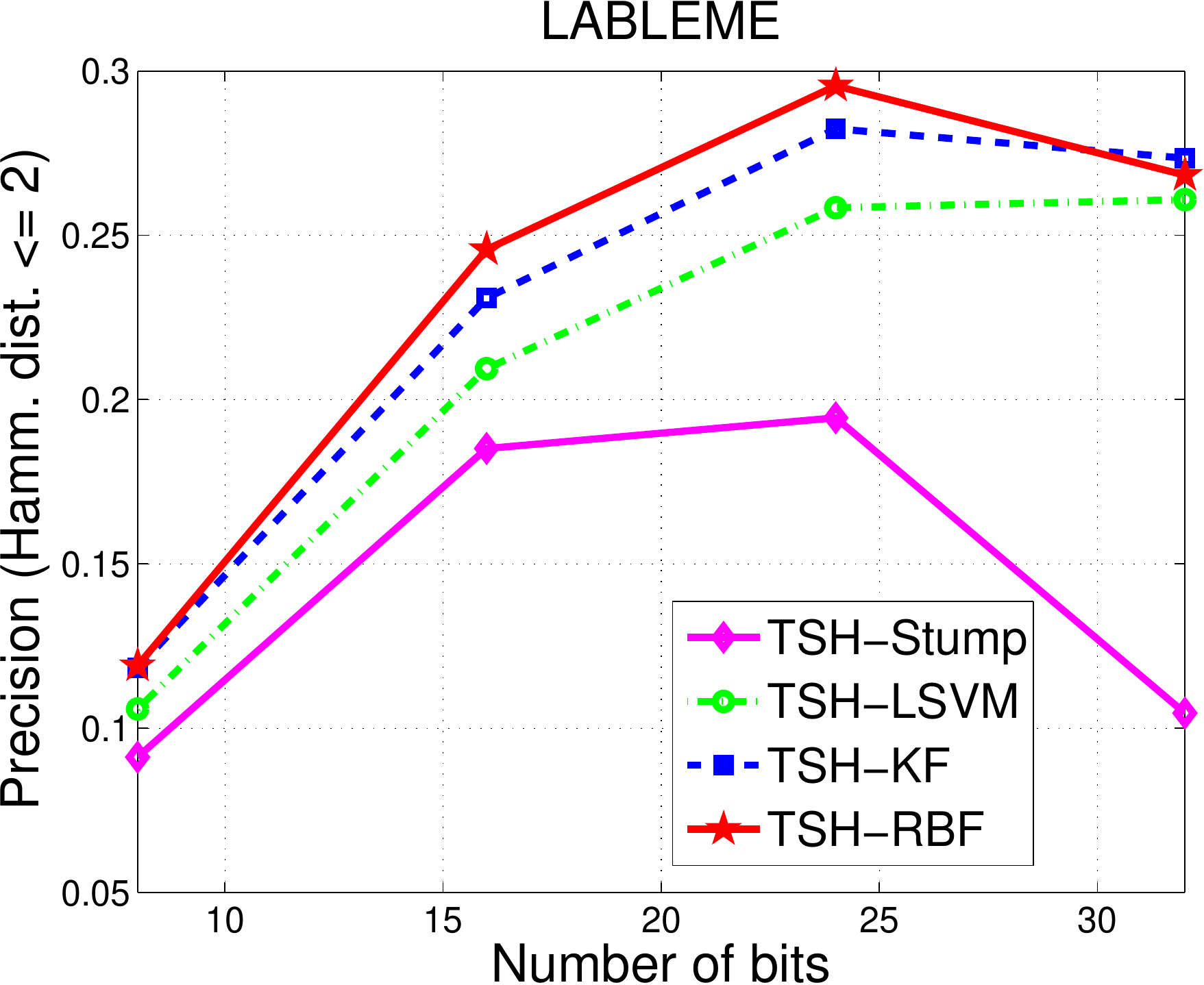}
   \includegraphics[width=.28\linewidth]{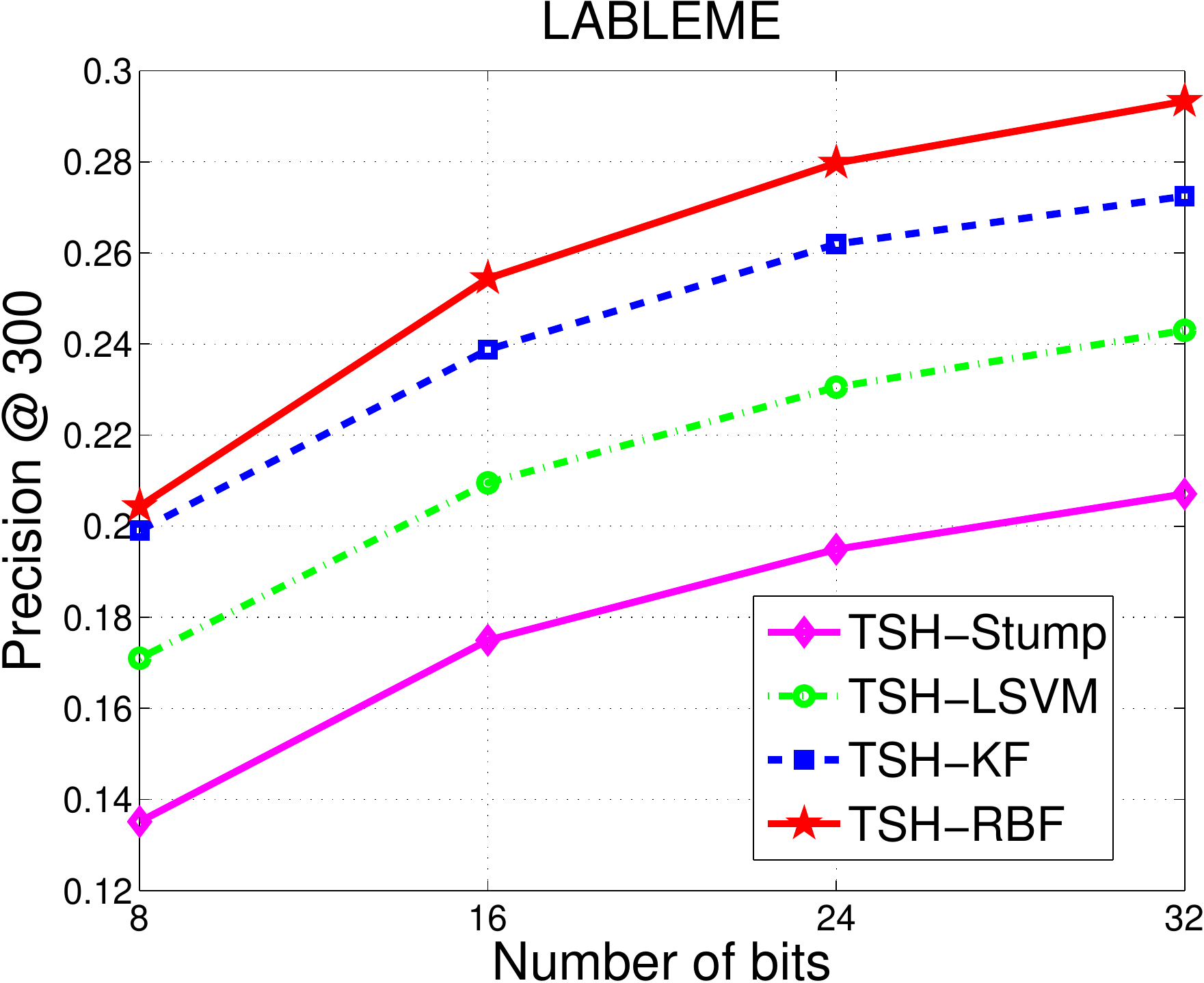}
   \includegraphics[width=.28\linewidth]{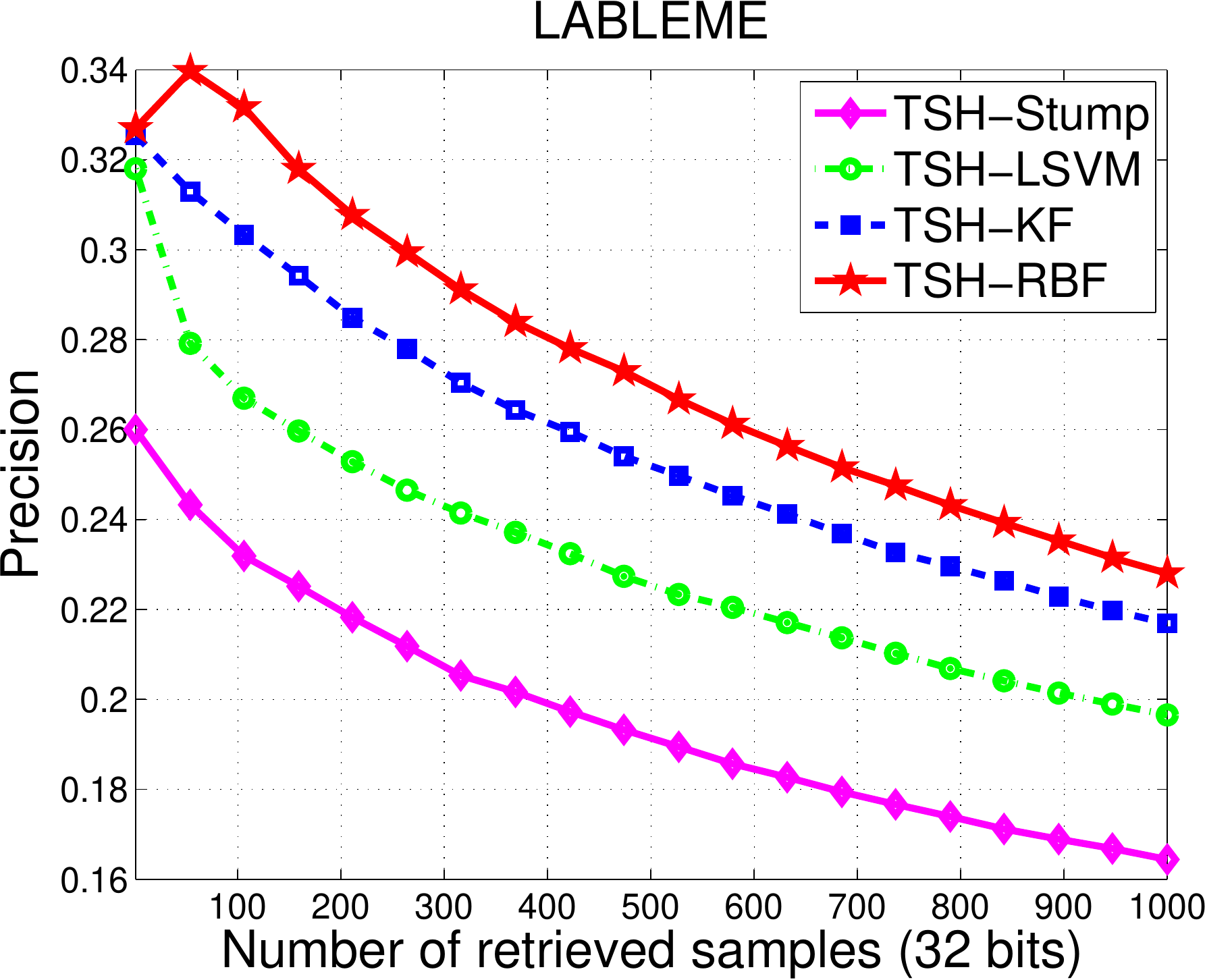}

   \includegraphics[width=.28\linewidth]{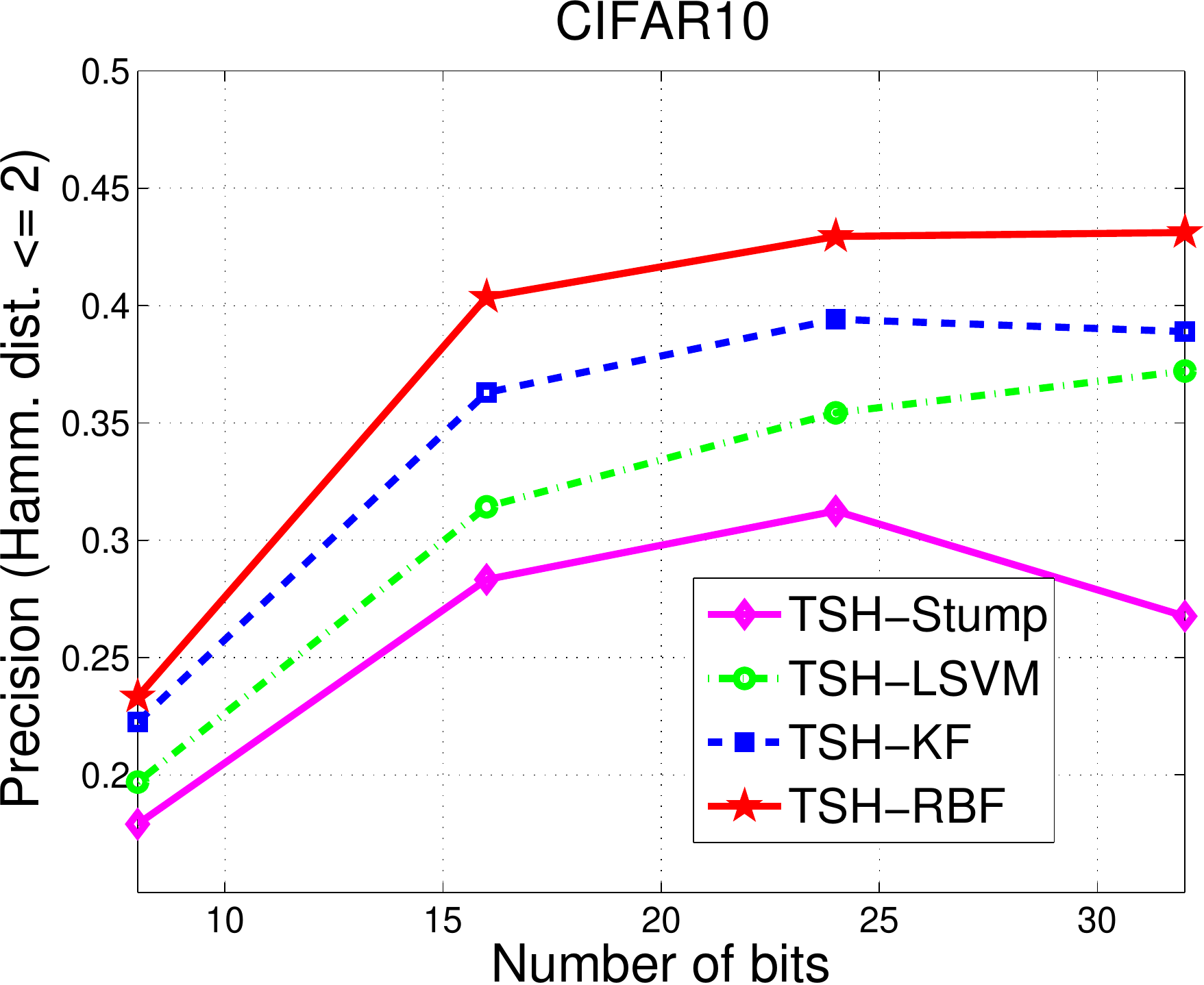}
   \includegraphics[width=.28\linewidth]{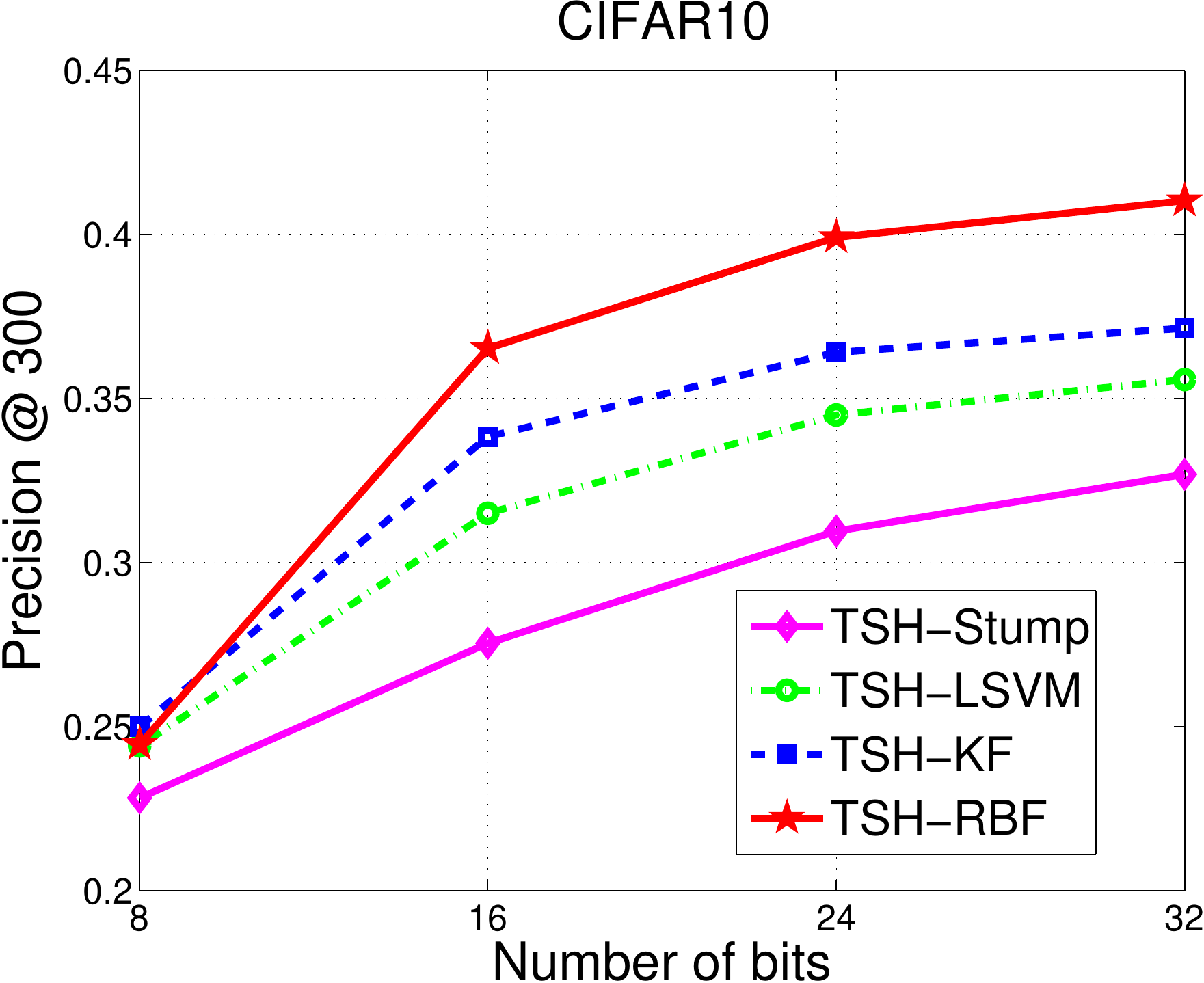}
   \includegraphics[width=.28\linewidth]{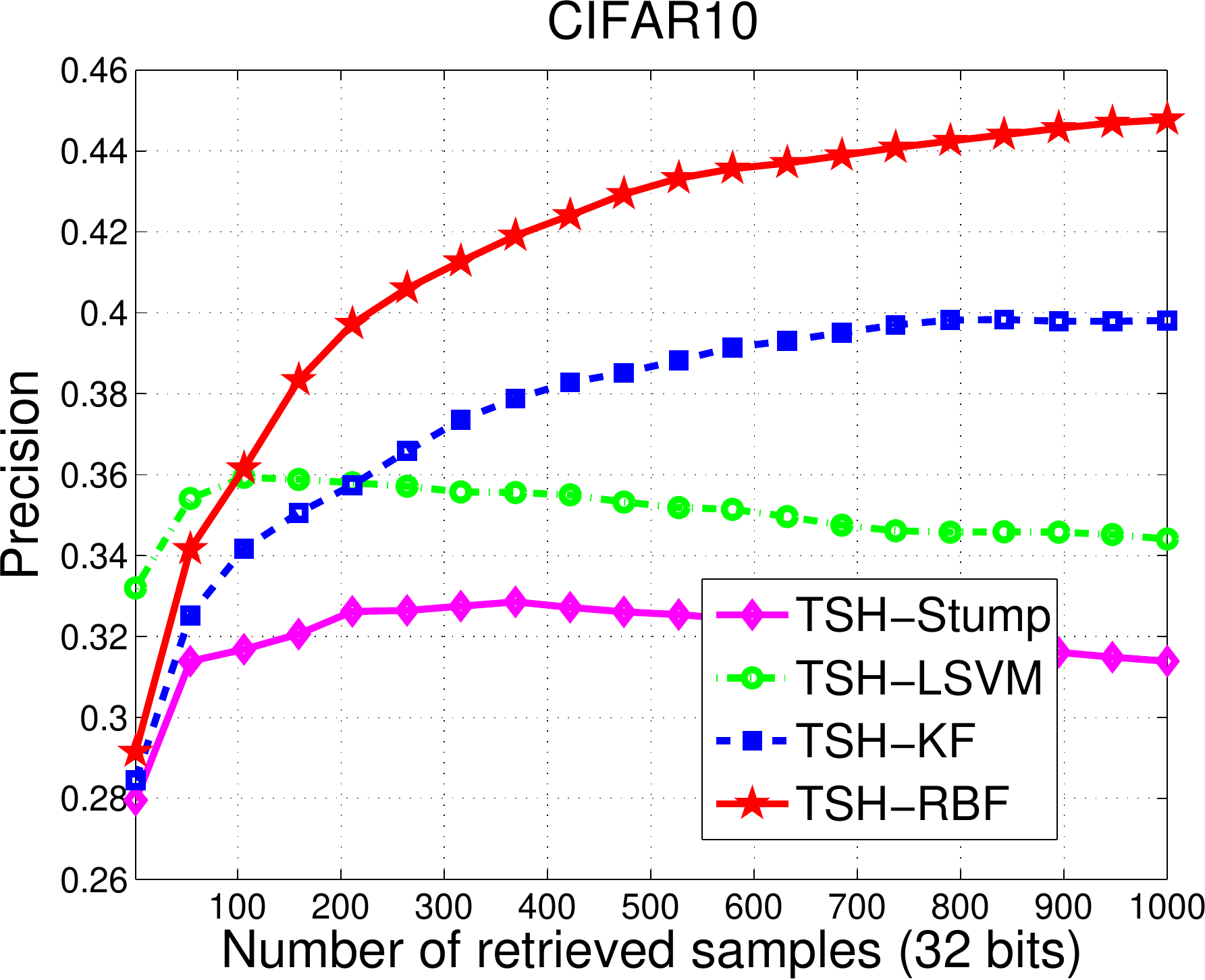}

    \caption{Results on 2 datasets of our method using different hash functions.
    Results show that using kernel hash function (TSH-RBF and TSH-KF) achieves best performances.}
   	\label{fig:wl}
\end{figure*}

\begin{figure}
    \centering

   \includegraphics[width=.349\linewidth]{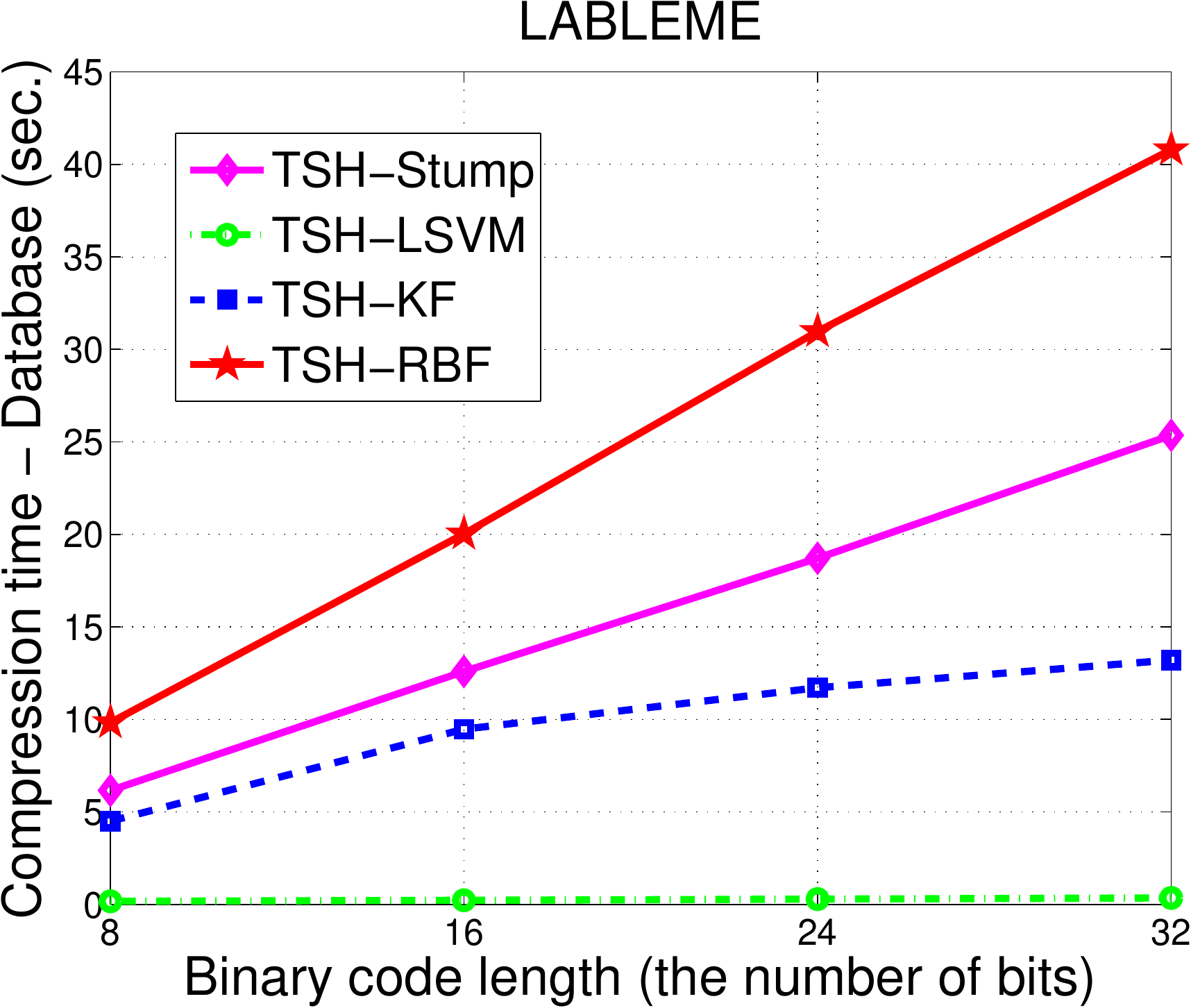}
   \includegraphics[width=.349\linewidth]{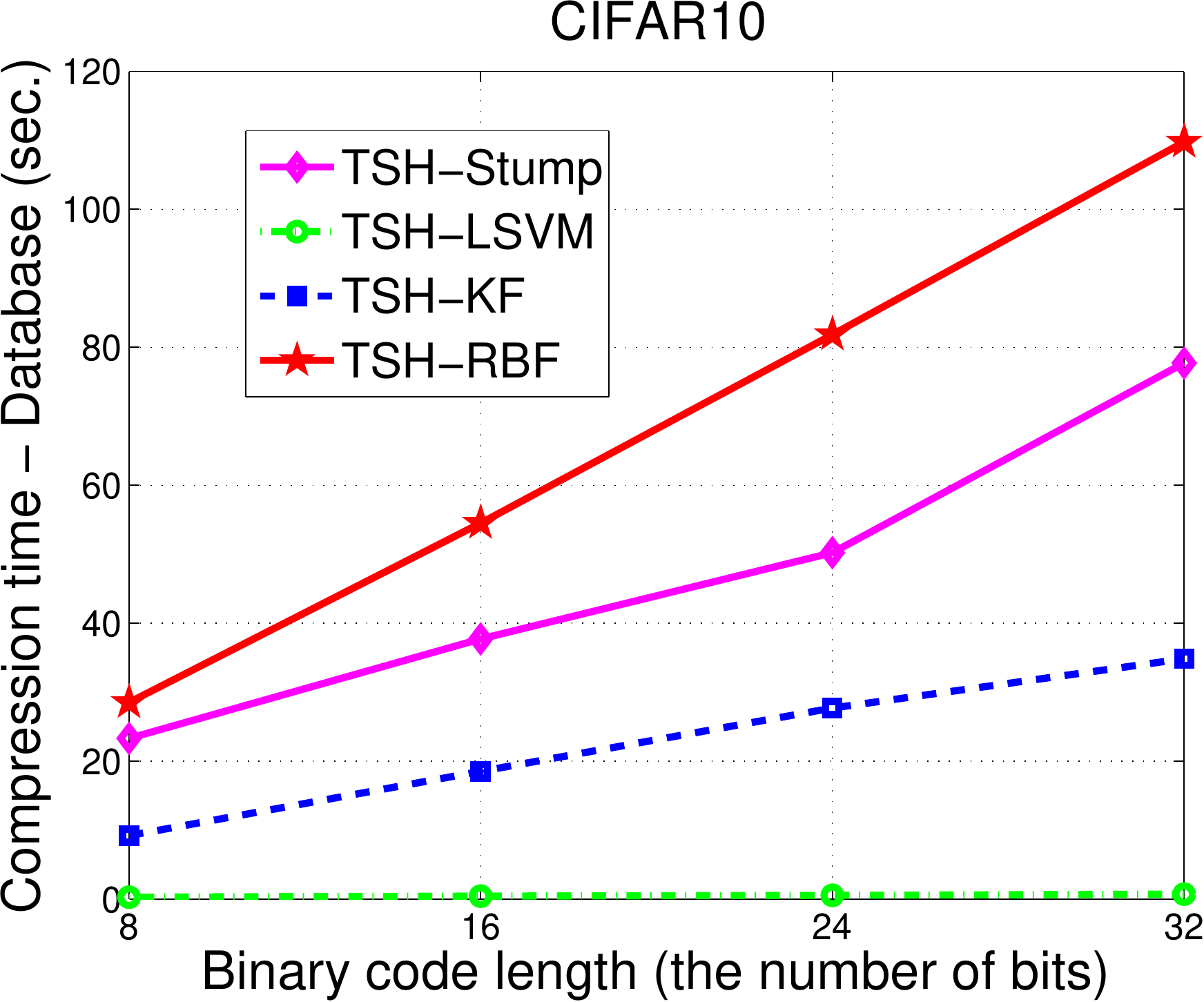}

    \caption{Code compression time  using
    different hash functions.
    Results show that using kernel transferred features (TSH-KF) is
    much faster than SVM with the RBF kernel (TSH-RBF). Linear SVM is the
    fastest one.
    }
   	\label{fig:wl_time}
\end{figure}

\begin{figure*}
    \centering

   \includegraphics[width=.3624\linewidth]{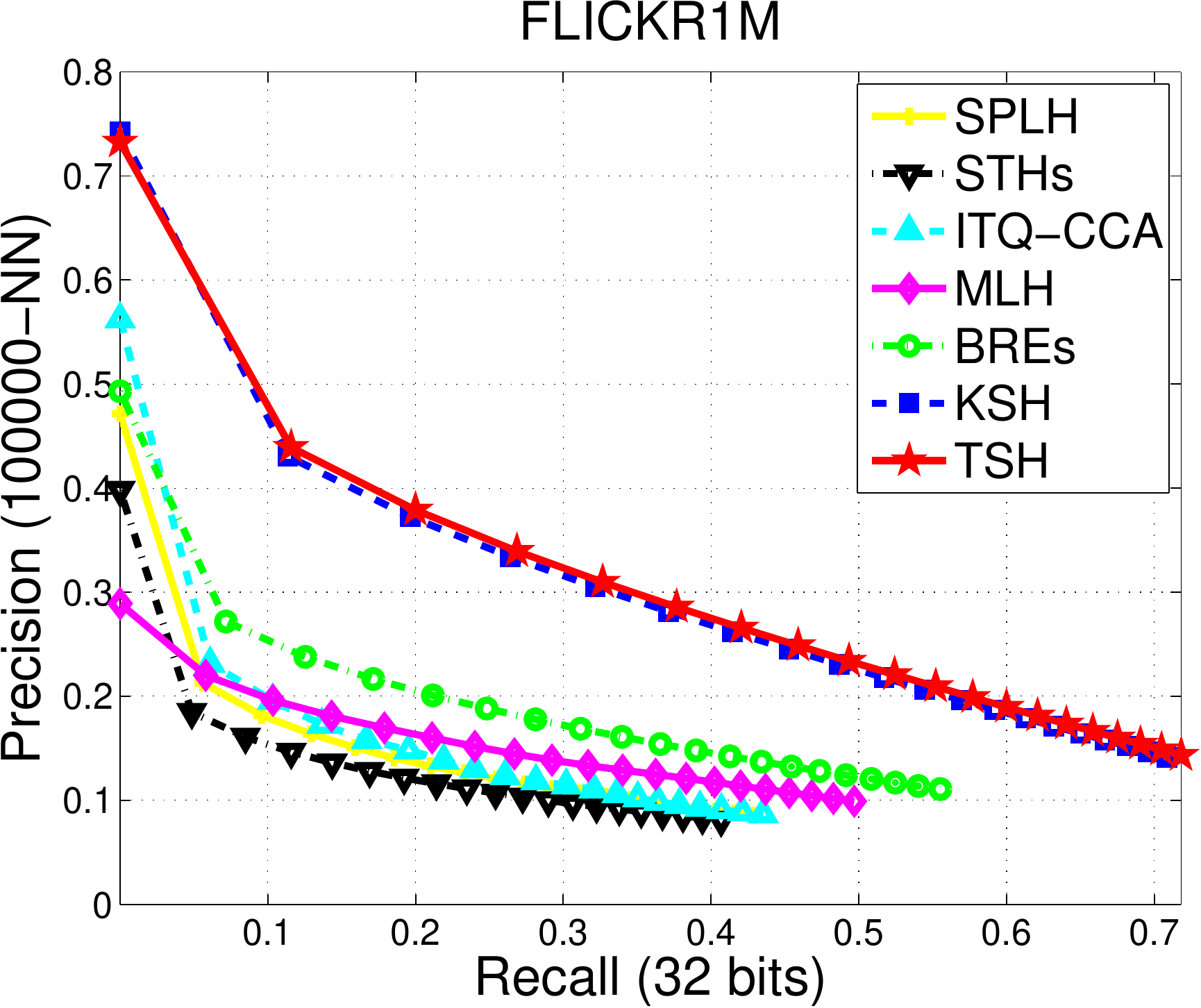}
   \includegraphics[width=.3624\linewidth]{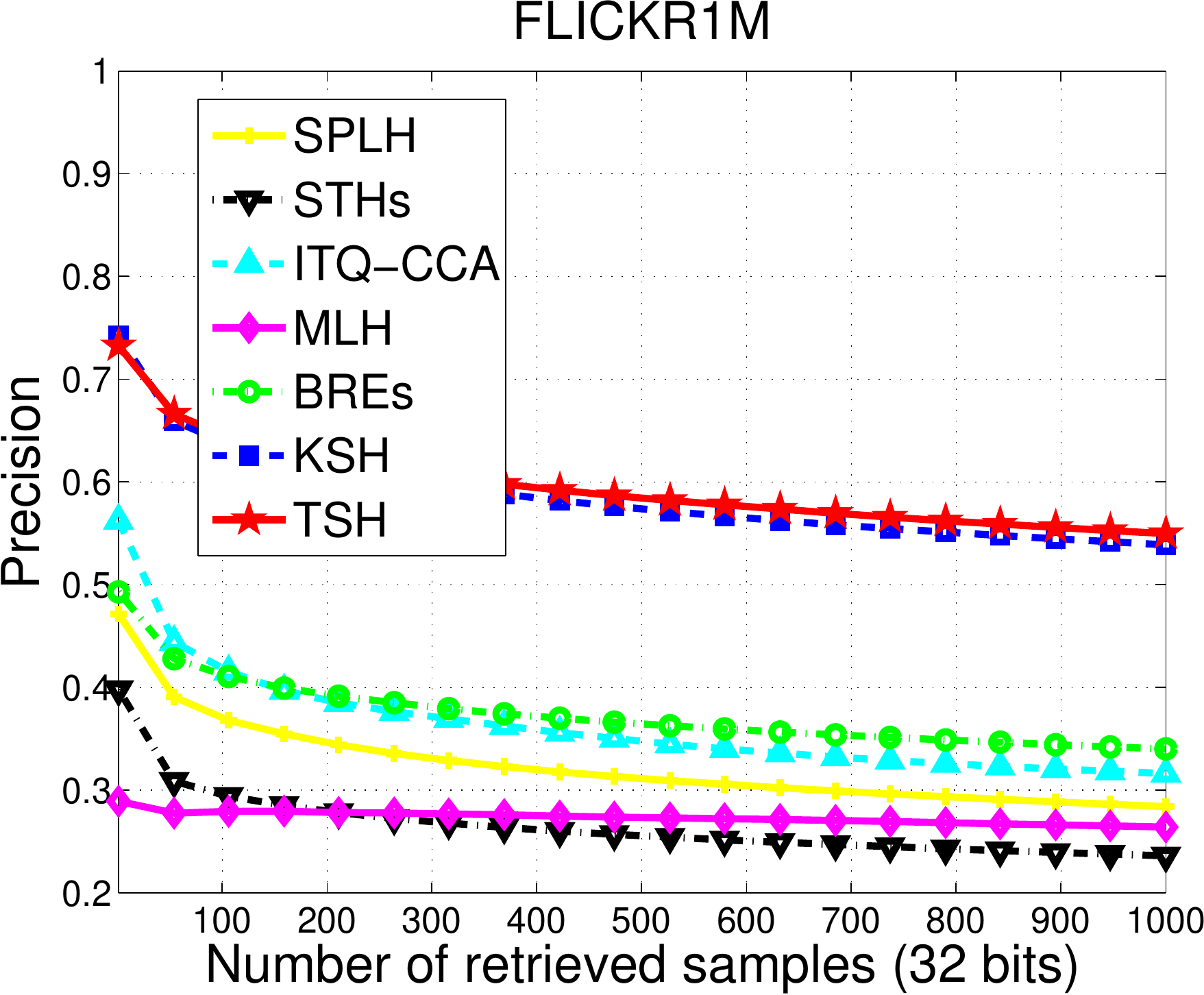}
   \includegraphics[width=.3624\linewidth]{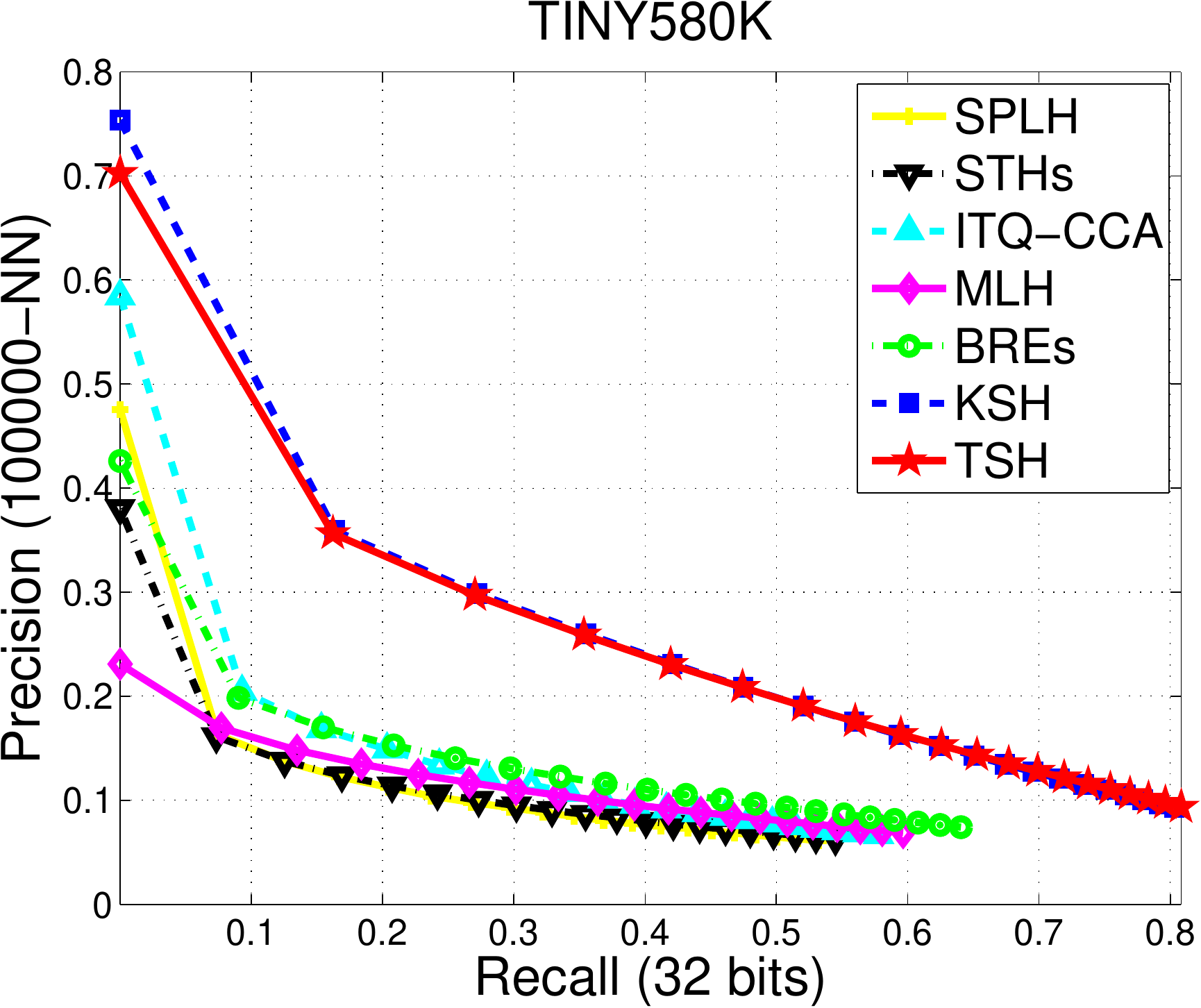}
   \includegraphics[width=.3624\linewidth]{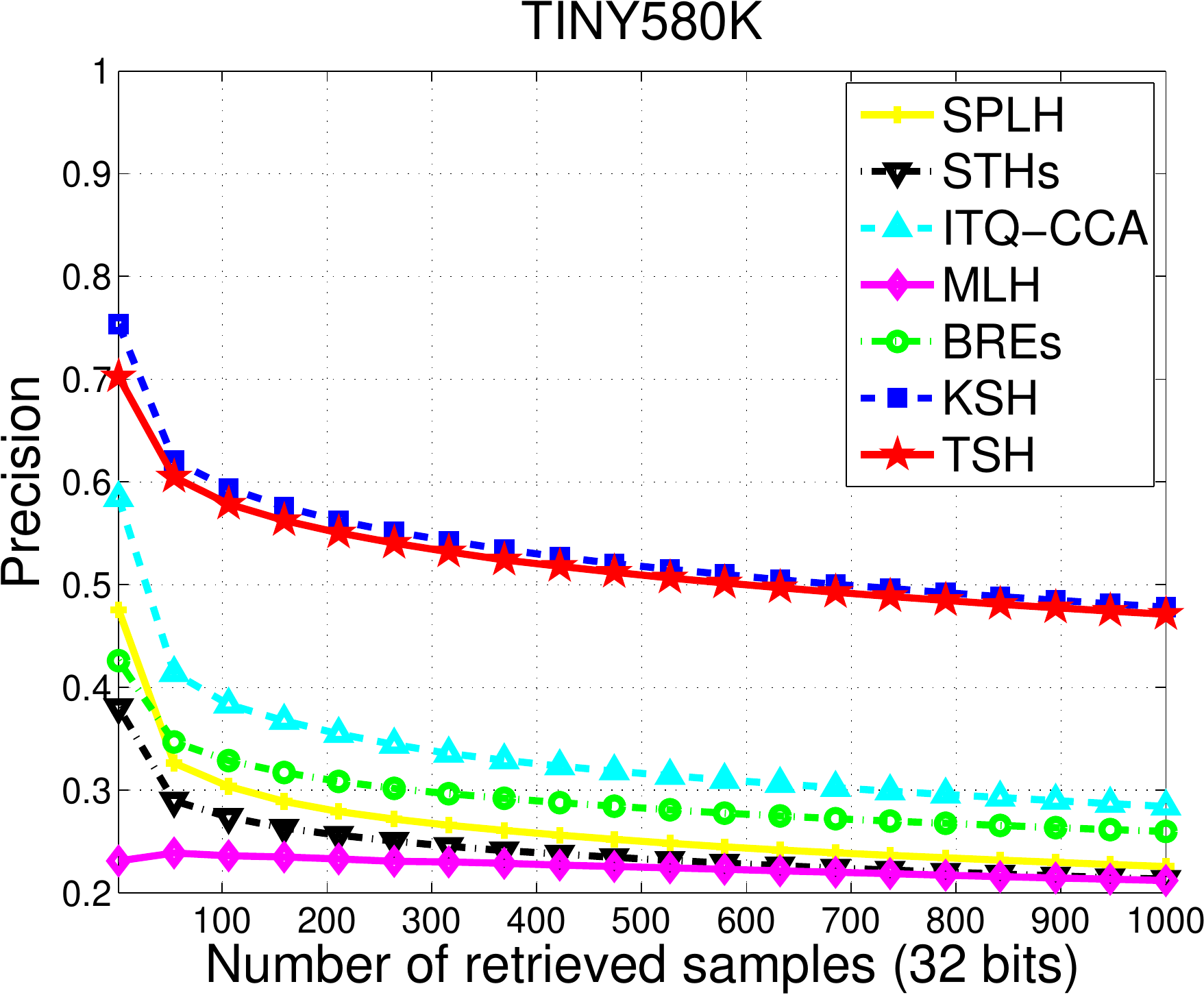}

   \includegraphics[width=.3624\linewidth]{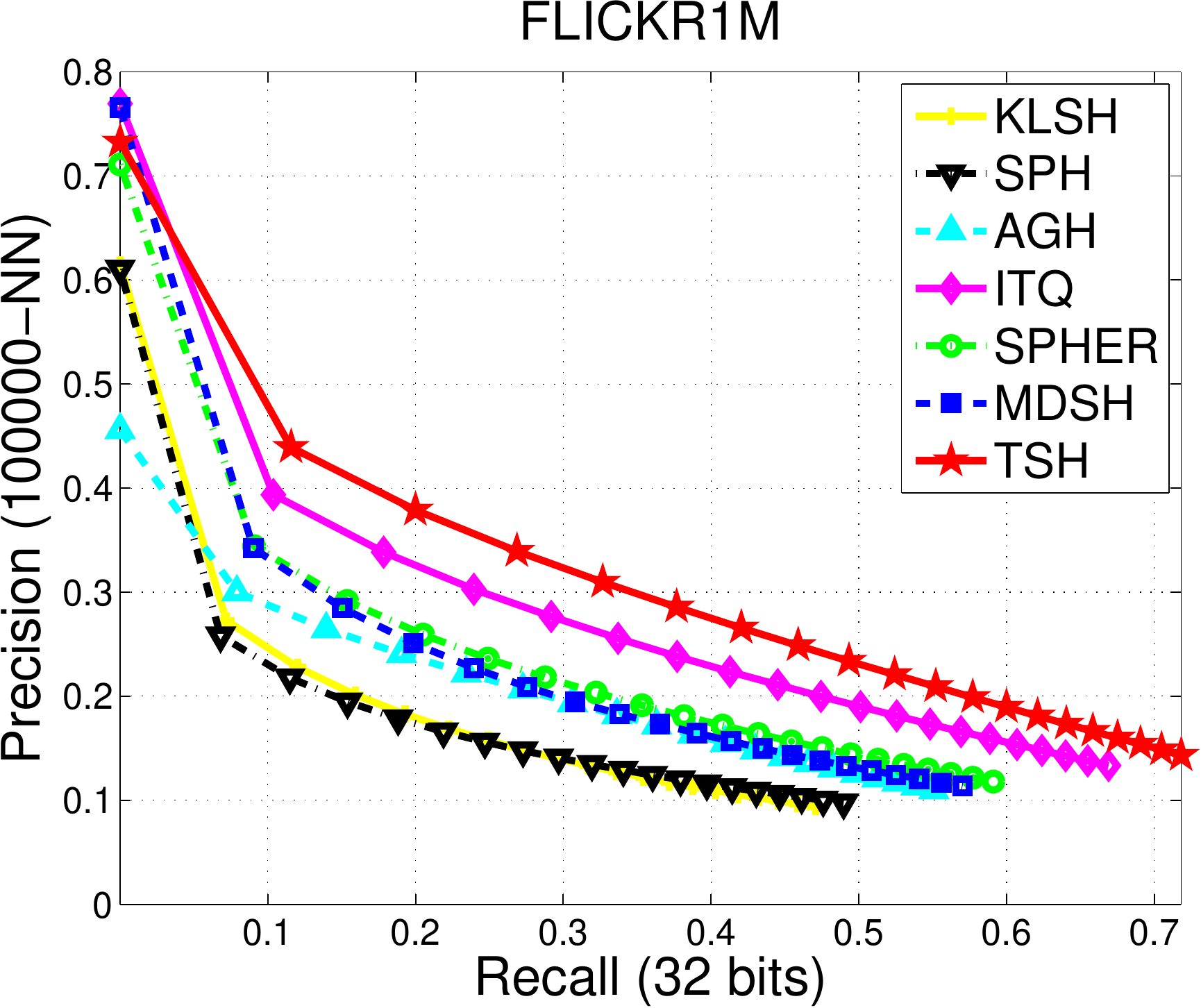}
   \includegraphics[width=.3624\linewidth]{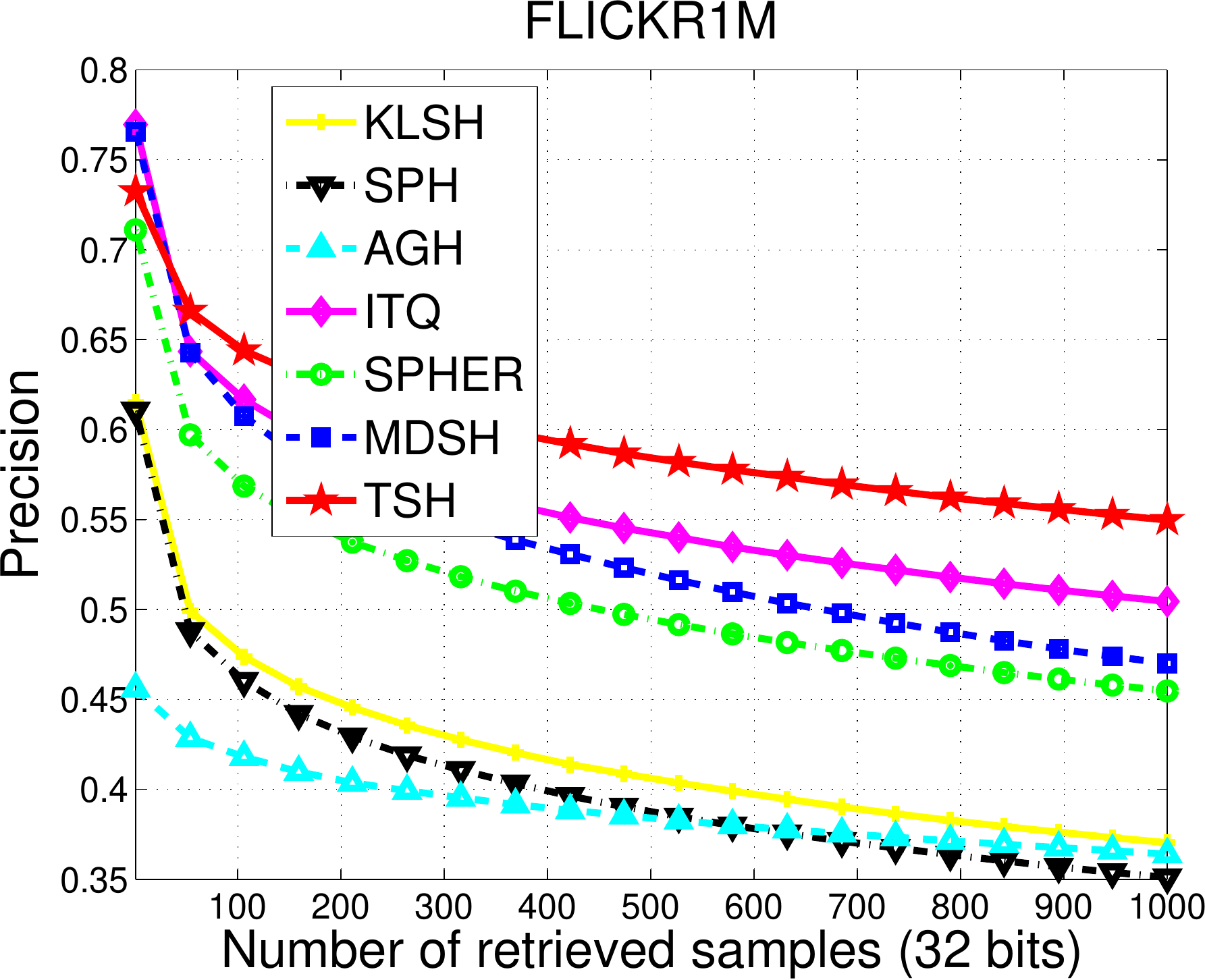}
   \includegraphics[width=.3624\linewidth]{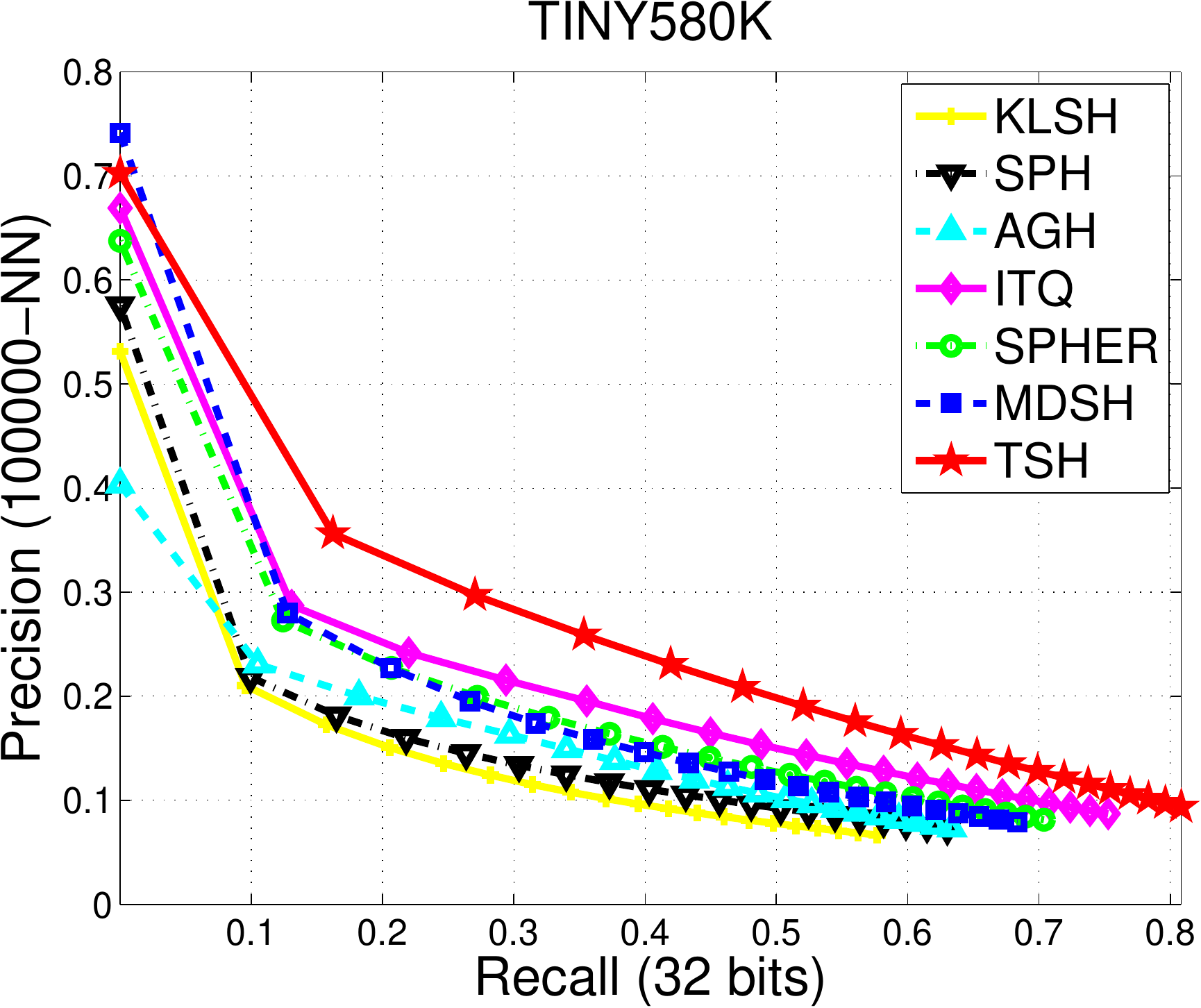}
   \includegraphics[width=.3624\linewidth]{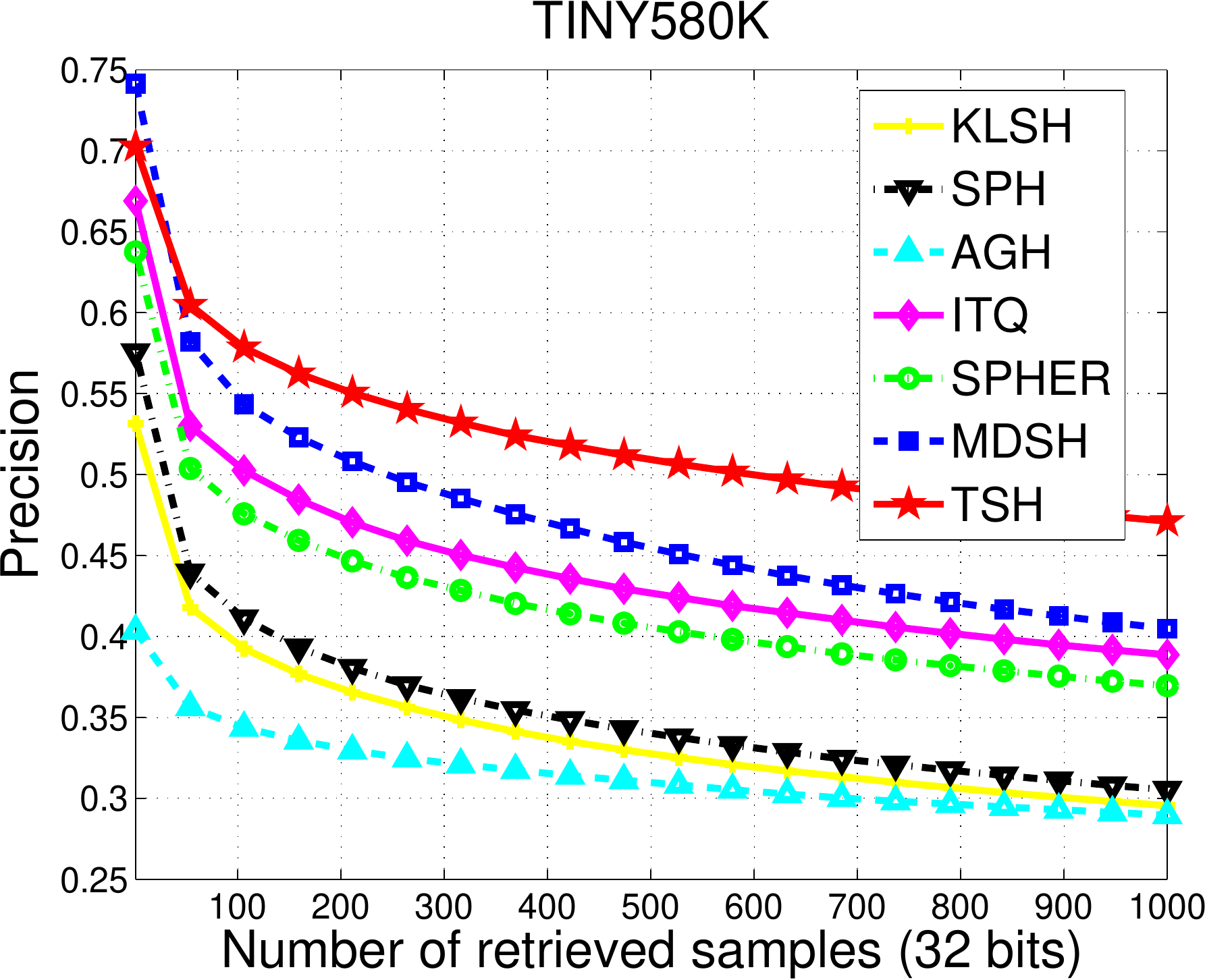}

    \caption{Results on 2 large-scale datasets: Flickr1M and Tiny580k.
    The top four plots show the results of supervised methods and the
    bottom four plots are the results of
    unsupervised methods. Our method TSH achieves on-par results with
    KSH, and TSH and KSH  significantly outperform other methods.
    }
   	\label{fig:large}
\end{figure*}

\textbf{Training time}
In Table\ \ref{tab:trn_time}, we compare the training time of
different methods.
It shows that our method is  fast compared to the
state-of-the-art. We also present the binary code learning time in
the table.
Notice that in the second step, learning hash functions by
binary classification can be easily paralleled which would make our
method even more efficient.

\subsection{Using different hash functions}
We evaluate our method using different hash functions. The hash
functions  are SVM with RBF kernel (TSH-RBF), linear SVM with kernel transferred feature (TSH-KF),
linear SVM (TSH-SVM), Adaboost with decision-stump (TSH-Stump).
Results on 3 datasets  are shown in Fig.\ \ref{fig:wl}. The testing
time for different hash functions are shown in Fig.\ \ref{fig:wl_time}.

It shows that the kernel hash functions (TSH-RBF and TSH-KF) achieve best performance in similarity search.
However, the testing  of linear hash functions is much faster than kernel hash functions.
We also find that the testing time of TSH-KF  is much faster then TSH-RBF.
The TSH-KF is a trade-off between testing time and search performance.

\subsection{Large datasets}

We carry out experiments on 2 large scale datasets: Flickr $1$ million
image dataset (Flickr1M) and $580,000$ Tiny image dataset (Tiny580k).
Results are shown in Fig.\ \ref{fig:large}.
Our method TSH achieve on par results with KSH. KSH and our TSH
significantly outperform other supervised or unsupervised methods.
Notice that there is no semantic similarity ground truth provided on
these two datasets. We generate the similarity ground truth using
the Euclidean distance.
 Some unsupervised methods are also able to perform
 well in this setting (e.g., MDSH, SPHER and ITQ).

\section{Conclusion}

We have shown that it is possible to place a wide variety of
learning-based hashing methods into a common framework, and that doing
so provides insight into the strengths, weaknesses, and commonality
between various competing methods.  One of the key insights arising is
the fact that the code generation and hash function learning processes
may be seen as separate steps, and that the latter may accurately be
composed as a classification problem.  This insight enables the
development of new approaches to hashing, one of which is detailed
above.  Experimental testing has validated this approach, and shown
that this new approach outperforms the state-of-the-art.

{
\bibliographystyle{ieee}
\bibliography{tsh}
}

\end{document}